\documentclass{article}
\usepackage[utf8]{inputenc}
\usepackage[T1]{fontenc}
\usepackage{amsmath, amssymb, amsthm, mathrsfs}
\usepackage{mathtools}
\usepackage{geometry}
\usepackage{algorithm}
\usepackage{algpseudocode}
\usepackage{authblk}

\newtheorem{definition}{Definition}[section]
\newtheorem{theorem}{Theorem}[section]

\newtheorem{proposition}{Proposition}[section]
\newtheorem{claim}{Claim}
\newtheorem{remark}{Remark}
\newtheorem{lemma}{Lemma}

\newcommand{\Sstate}{\mathcal{S}}
\newcommand{\Action}{\mathcal{A}}
\newcommand{\R}{\mathbb{R}}
\newcommand{\Ptwo}{\mathcal{P}_2}
\newcommand{\Borel}{\mathscr{B}}

\title{Wasserstein Formulation of Reinforcement Learning:\\ \large An Optimal Transport Perspective on Policy Optimization}
\author{Mathias Dus\thanks{Institut de Recherche Mathématique Avancée, Université de Strasbourg (e-mail : \texttt{mathias.dus@math.unistra.fr}).}}
\date{\today}

\begin{document}

\maketitle

\begin{abstract}
 We present a geometric framework for Reinforcement Learning (RL) that views policies as maps into the Wasserstein space of action probabilities. First, we define a Riemannian structure induced by stationary distributions, providing rigorous guarantees for their existence in a general context. We then define the tangent space of policies and characterize the geodesics, specifically addressing the measurability of vector fields mapping from the state space to the tangent space of probability measures over the action space. Next, we formulate a general RL optimization problem and construct a gradient flow using Otto's calculus. We compute the gradient and the Hessian of the energy (i.e., the expected cumulative cost), providing a formal second-order analysis. Finally, we illustrate the method with numerical examples: we compute the gradient directly from our theoretical formalism for low-dimensional problems, and we illustrate the numerical applicability of our formal framework to high-dimensional continuous control by parameterizing the policy with a neural network optimized via an ergodic approximation of the cost.
\end{abstract}

\section{Introduction}
\label{sec:related_work}
At the heart of policy optimization lies a fundamental geometric question: how should we measure the distance between two policies in order to define a gradient flow that induces a monotonic decrease of the cost? Answering this question has naturally drawn reinforcement learning (RL) into a fruitful dialogue with information geometry and optimal transport. While existing works have begun exploring this intersection, they often remain limited to first-order approximations or specific parametric families. In this work, we generalize these approaches by developing a comprehensive Wasserstein policy optimization framework, combining first-order gradient flows with a detailed second-order Riemannian variational analysis. 

\paragraph{Information Geometry and Trust Regions in RL.} 
Viewing policy optimization through a geometric lens has a long history in RL. Standard approaches often rely on the Fisher-Rao metric and the Kullback-Leibler (KL) divergence to define the geometry of the policy space. This perspective led to seminal algorithms such as Natural Policy Gradient (NPG) \cite{kakade2001natural} and Trust Region Policy Optimization (TRPO) \cite{schulman2015trust}, which restricts policy updates within a KL-based trust region and provides theoretical guarantees for policy improvement under suitable assumptions. While KL divergence is computationally convenient, it does not capture the underlying geometry of the action space, which can lead to inefficient updates when the action space possesses a natural metric structure (e.g., continuous control).

\paragraph{Wasserstein Distances in Policy Optimization.}

The limitations of KL-based geometries have motivated the use of Optimal Transport (OT) in reinforcement learning, particularly when the action space is endowed with a meaningful metric structure. Unlike KL divergence, the Wasserstein distance explicitly accounts for this underlying geometry and remains informative even when probability distributions have disjoint supports. In policy optimization, Wasserstein distances have primarily been introduced through behavioral regularization and trust-region constructions. For instance, \cite{pacchiano2020wasserstein} compares policy behaviors in a latent behavioral space using Wasserstein distances and incorporates the resulting discrepancy as a regularizer in policy optimization. These approaches use Wasserstein geometry to constrain or regularize policy updates, rather than to define the optimization dynamics themselves.

\paragraph{Policy Optimization as Wasserstein Gradient Flows.}

A more geometric perspective arises by interpreting policy optimization itself as a gradient flow on a Wasserstein space. Building on the general theory of Wasserstein gradient flows \cite{otto2001geometry,ambrosio2008gradient}, \cite{zhang2018policy} formulated policy optimization in terms of Wasserstein gradient flows, considering both policy-distribution and parameter-distribution perspectives. Related work has subsequently considered more structured policy classes: \cite{ziesche2023wasserstein} studied Wasserstein gradient flows for Gaussian Mixture Models (GMMs) using the Bures-Wasserstein geometry. More recently, \cite{pfau2025wasserstein} developed Wasserstein Policy Optimization (WPO), deriving an actor-critic method by approximating a continuous-time Wasserstein gradient flow over the space of policies and projecting the resulting dynamics onto neural-network parameters.

\paragraph{Positioning and Contributions of Our Work.}
While previous works have explored Wasserstein gradient flows
for policy optimization, they consider different settings and
approximations. For example, \cite{zhang2018policy} studies
policy-distribution and parameter-distribution formulations,
\cite{ziesche2023wasserstein} focuses on Gaussian mixture
policies, and \cite{pfau2025wasserstein} derives a practical
actor-critic approximation by projecting a Wasserstein gradient
flow onto finite-dimensional policy parameters.  Here, our objective is not to design an algorithm aimed at achieving state-of-the-art empirical performance on standard RL benchmarks, but rather to construct a solid geometric foundation for continuous Wasserstein policy optimization, combining rigorous ergodic and measurability results with formal Otto calculus across four theoretical pillars:
\begin{enumerate}
\item Our framework relies on a reference measure that dynamically depends on the policy itself: the stationary distribution. To construct a non-degenerate Riemannian metric weighted by this distribution, one must first rigorously guarantee its existence and uniqueness. To bridge this gap for the optimal transport and applied mathematics communities, we provide a self-contained pedagogical review of classical ergodic theory adapted to Markov Decision Processes. By explicitly outlining the necessary structural assumptions—ranging from Wasserstein strict contractions to Weak/Strong Feller properties and Doeblin's condition—we clarify the exact mathematical prerequisites required to ensure that our subsequent Wasserstein geometric constructions and tangent vector fields are strictly well-posed.
	
    \item \textbf{Rigorous State-Conditional Geometry:} We formally define the Riemannian structure induced by stationary distributions and explicitly address the \textit{measurability} of vector fields mapping from the state space to the tangent space of probability measures over the action space. This provides a formal foundation often overlooked in the RL literature.
    
    \item \textbf{Second-Order Variational Analysis via Otto's Calculus:} Unlike existing works that stop at the first-order gradient \cite{zhang2018policy, pfau2025wasserstein}, we utilize Otto's calculus \cite{otto2001geometry, ambrosio2008gradient} to compute both the Riemannian gradient and the \textit{Hessian} of the expected cost along policy geodesics. This second-order analysis reveals the underlying geometric structure of the objective functional and explicitly characterizes the non-local curvature induced by the environment dynamics, providing a formal tool to study geodesic convexity.
    
    \item \textbf{Ergodic Approximation for Scalability:} While we provide exact gradient computations for low-dimensional cases based on our theoretical formalism, we bridge the gap to high-dimensional problems by parameterizing the policy with a neural network optimized via an \textit{ergodic approximation} of the cost, serving as a numerical proof-of-concept. These experiments are designed to illustrate that our exact theoretical gradients can be effectively evaluated, even in high-dimensional continuous control, rather than to establish empirical supremacy over standard RL solvers.
\end{enumerate}

The paper is organized as follows. Section \ref{sec:Proba_dyn_pol_space} rigorously defines the policy space within a metric space framework. We introduce the concept of a transition kernel induced by a policy and define its corresponding invariant measures. Additionally, we provide a self-contained overview of theorems from the theory of Markov chains in metric spaces to establish existence and uniqueness results for these invariant measures. In Section \ref{sec:wass_geom_optim}, we propose a Riemannian geometric interpretation of the policy space, utilizing the unique invariant measure to define the metric. Using Otto's calculus, we compute the gradient of the conventional long-term cost that is minimized in reinforcement learning. Section \ref{sec:Hessian} details the calculation of the Hessian for this energy and analyzes its geodesic convexity through a simplified case. Finally, Section \ref{sec_num_analysis} presents numerical examples of the method in low-dimensional environments and adapts the technique to high-dimensional settings using neural networks.

\section{Probabilistic Dynamics and Policy Space}\label{sec:Proba_dyn_pol_space}

We formalize the Reinforcement Learning (RL) problem as an optimization problem over the space of measurable functions mapping the state space to the Wasserstein space of action distributions.

\subsection{State and Action Spaces}

Let $(\Sstate, d_\Sstate)$ be a Polish space (complete separable metric space) representing the \textbf{state space}. Let $(\Action, d_\Action)$ be a Polish space representing the \textbf{action space}. We denote by $\Ptwo(\Action)$ the  space of probability measures on $\Action$, defined as:
\[
    \Ptwo(\Action) := \left\{ \nu \in \mathcal{P}(\Action) \;\middle|\; \int_{\Action} d_\Action(a_0, a)^2 \, \mathrm{d}\nu(a) < \infty \right\},
\]
where $a_0 \in \Action$ is an arbitrary reference point.

The space $\Ptwo(\Action)$ is equipped with the $2$-Wasserstein metric $W_2$, defined for any $\mu, \nu \in \Ptwo(\Action)$ by:
\[
    W_2^2(\mu, \nu) := \inf_{\gamma \in \Gamma(\mu, \nu)} \int_{\Action \times \Action} d_\Action(x, y)^2 \, \mathrm{d}\gamma(x, y),
\]
where $\Gamma(\mu, \nu)$ is the set of transport plans (couplings) between $\mu$ and $\nu$.

\subsection{The Policy Space}

In this geometric framework, we view a policy as a measurable map from the state space to the Wasserstein space of action distributions $\mathcal{P}_2(\mathcal{A})$

\begin{definition}[Policy Space]
Let $\Borel(\Sstate)$ denote the Borel $\sigma$-algebra of $\Sstate$. The space of admissible policies, denoted by $\Pi$, is the space of Borel measurable functions from $\Sstate$ to $\Ptwo(\Action)$:
\[
    \Pi := L^0(\Sstate; \Ptwo(\Action)) = \{ \pi : \Sstate \to \Ptwo(\Action) \mid \pi \text{ is } (\Borel(\Sstate), \Borel(\Ptwo(\Action)))\text{-measurable} \}.
\]
\end{definition}

\subsection{Induced Dynamics and Ergodicity}\label{sec:induced_dyn_erg}
To define a Riemannian metric—and the resulting gradient flow—weighted by the state occupancy of the policy, we must first ensure the existence and uniqueness of a stationary distribution.

\begin{definition}[Induced Transition Kernel]
Let $\mathcal{T}: \Sstate \times \Action \times \Borel(\Sstate) \to [0, 1]$ be the environment transition kernel. For any policy $\pi \in \Pi$, the \textbf{induced state-transition kernel} $P^\pi: \Sstate \times \Borel(\Sstate) \to [0, 1]$ is defined by integrating out the action distribution:
\[
    P^\pi(s, B) := \int_{\Action} \mathcal{T}(s, a, B) \, \mathrm{d}(\pi(s))(a), \quad \forall s \in \Sstate, \forall B \in \Borel(\Sstate).
\]
\end{definition}

\noindent \textbf{Dual Operator Formulation.}
We view the transition kernel $P^\pi$ as inducing two linear operators acting on dual Banach spaces:
\begin{enumerate}
    \item \textbf{Transition Operator (Action on Functions):} For any bounded measurable function $f \in L^\infty(\Sstate)$, $P^\pi$ acts by averaging over future states:
    \[
        (P^\pi f)(s) := \int_{\Sstate} f(s') \, P^\pi(s, \mathrm{d}s').
    \]
    \item \textbf{Transfer Operator (Action on Measures):} For any finite signed measure $\nu \in \mathcal{M}(\Sstate)$, the operator acts on the left (or via the adjoint $P^{\pi*}$) to evolve the distribution:
    \[
        (\nu P^\pi)(B) := \int_{\Sstate} P^\pi(s, B) \, \mathrm{d}\nu(s).
    \]
\end{enumerate}
These operators are adjoints with respect to the standard bilinear form. For any measure $\nu$ and function $f$, we denote the pairing $\langle \nu, f \rangle_{\Sstate} := \int_\Sstate f \, \mathrm{d}\nu$. The defining adjoint relation is:
\[
    \langle \nu P^\pi, f \rangle_{\Sstate} = \langle \nu, P^\pi f \rangle_{\Sstate}.
\]

A probability measure $\mu_\pi \in \mathcal{P}(\Sstate)$ is called an \textbf{invariant measure} (or stationary distribution) for $\pi$ if it satisfies the stationarity equation:
\[
    \mu_\pi(B) = \int_{\Sstate} P^\pi(s, B) \, \mathrm{d}\mu_\pi(s), \quad \forall B \in \Borel(\Sstate).
\]

To guarantee the well-posedness of $\mu_\pi$, usually we rely on contraction properties.

\begin{definition}[Lipschitz Continuity in Wasserstein Space]
We say that the environment dynamics $\mathcal{T}$ are $(L_\Sstate, L_\Action)$-Lipschitz if for all $s, s' \in \Sstate$ and $a, a' \in \Action$:
\[
    W_2(\mathcal{T}(s, a, \cdot), \mathcal{T}(s', a', \cdot)) \leq L_\Sstate d_\Sstate(s, s') + L_\Action d_\Action(a, a').
\]
Similarly, a policy $\pi \in \Pi$ is $K_\pi$-Lipschitz if:
\[
    W_2(\pi(s), \pi(s')) \leq K_\pi d_\Sstate(s, s').
\]
\end{definition}
The existence and uniqueness of a stationary measure in the context of Lipschitz Markov Decision Processes (MDPs) rely on standard contraction arguments. Similar properties for the transition kernel have been extensively studied in the literature, notably by \cite{pirotta2015policy, asadi2018lipschitz}, often relying on the 1-Wasserstein distance ($W_1$). In Theorem \ref{th:existence_invariant}, we adapt this result to the $W_2$ distance and provide a detailed, self-contained proof of the strict contraction, a crucial step for constructing our gradient flow.

\begin{theorem}\label{th:existence_invariant}
Assume that the environment dynamics and the policy satisfy the Lipschitz conditions defined above. If the combined contraction coefficient satisfies:
\[
    \kappa := L_\Sstate + L_\Action K_\pi < 1,
\]
then the induced transition operator $\Phi_\pi: \nu \mapsto \nu P^\pi$ is a strict contraction on the complete metric space $(\mathcal{P}_2(\Sstate), W_2)$.
Consequently, there exists a unique invariant measure $\mu_\pi \in \mathcal{P}_2(\Sstate)$, and for any initial distribution $\mu_0$, the sequence converges exponentially fast:
\[
    W_2(\mu_0 (P^\pi)^t, \mu_\pi) \leq \frac{\kappa^t}{1-\kappa} W_2(\mu_0 (P^\pi), \mu_0).
\]
\end{theorem}
\begin{proof}
    Let $\nu_1, \nu_2 \in \mathcal{P}_2(\Sstate)$. We aim to bound the distance $W_2(\nu_1 P^\pi, \nu_2 P^\pi)$.
    Let $\xi \in \Gamma_o(\nu_1, \nu_2)$ be an optimal coupling between the initial state distributions, such that:
    \[
        \int_{\Sstate \times \Sstate} d_\Sstate(s, s')^2 \, \mathrm{d}\xi(s, s') = W_2^2(\nu_1, \nu_2).
    \]
    
    \textbf{Step 1: Pointwise bound on transition kernels.}
    We need to prove the following claim:
\begin{claim}
Let $s, s' \in \mathcal{S}$ be two states. Then, the following inequality holds:
\[
    W_2^2(P^\pi(s, \cdot), P^\pi(s', \cdot)) \leq \inf_{\gamma \in \Gamma(\pi(s), \pi(s'))} \int_{\mathcal{A} \times \mathcal{A}} W_2^2(\mathcal{T}(s, a, \cdot), \mathcal{T}(s', a', \cdot)) \, \mathrm{d}\gamma(a, a').
\]
\end{claim}

\begin{proof}
 To prove the upper bound on $W_2^2(P^\pi(s, \cdot), P^\pi(s', \cdot))$, we explicitly construct a specific coupling (transport plan) between $P^\pi(s, \cdot)$ and $P^\pi(s', \cdot)$ and evaluate its cost.

\paragraph{Coupling of Policies.}
Let $\gamma \in \Gamma(\pi(s), \pi(s'))$ be an arbitrary coupling between the policy distributions at states $s$ and $s'$. This coupling satisfies marginal constraints:
\[
    \int_{\mathcal{A}} \mathrm{d}\gamma(a, a') = \mathrm{d}\pi(s')(a') \quad \text{and} \quad \int_{\mathcal{A}'} \mathrm{d}\gamma(a, a') = \mathrm{d}\pi(s)(a).
\]

\paragraph{Coupling of Dynamics.}
For every pair of actions $(a, a') \in \mathcal{A} \times \mathcal{A}$, let $\Pi_{(s,a), (s',a')}$ denote an optimal transport plan between the transition dynamics $\mathcal{T}(s, a, \cdot)$ and $\mathcal{T}(s', a', \cdot)$. By definition, the cost of this specific plan yields the Wasserstein distance between the transitions:
\[
    \int_{\mathcal{S} \times \mathcal{S}} d_{\Sstate} (x,y)^2 \, \mathrm{d}\Pi_{(s,a), (s',a')}(x, y) = W_2^2(\mathcal{T}(s, a, \cdot), \mathcal{T}(s', a', \cdot)).
\]

\paragraph{Construction of the Mixture Coupling.}
We define a global candidate coupling $\mathcal{K}$ on $\mathcal{S} \times \mathcal{S}$ by integrating the coupling $\Pi$ with the policy coupling $\gamma$:
\[
    \mathcal{K}(\mathrm{d}x, \mathrm{d}y) = \int_{\mathcal{A} \times \mathcal{A}} \Pi_{(s,a), (s',a')}(\mathrm{d}x, \mathrm{d}y) \, \mathrm{d}\gamma(a, a').
\]
Note that this integration is well-defined since \cite[Corollary 5.22]{villani2009optimal} guarantees the existence of a Borel measurable selection of optimal plans $(a, a') \mapsto \Pi_{(s,a), (s',a')}$.

First, we verify that $\mathcal{K}$ is a valid coupling between $P^\pi(s, \cdot)$ and $P^\pi(s', \cdot)$. Checking the first marginal:
\begin{align*}
    \mathcal{K}(\mathrm{d}x, \mathcal{S}) &= \int_{\mathcal{A} \times \mathcal{A}} \Pi_{(s,a), (s',a')}(\mathrm{d}x, \mathcal{S}) \, \mathrm{d}\gamma(a, a') \\
    &= \int_{\mathcal{A} \times \mathcal{A}} \mathcal{T}(s, a, \mathrm{d}x) \, \mathrm{d}\gamma(a, a') \quad (\text{since } \Pi \text{ has marginal } \mathcal{T}) \\
    &= \int_{\mathcal{A}} \mathcal{T}(s, a, \mathrm{d}x) \left( \int_{\mathcal{A}} \mathrm{d}\gamma(a, a') \right) \\
    &= \int_{\mathcal{A}} \mathcal{T}(s, a, \mathrm{d}x) \, \mathrm{d}\pi(s)(a) \\
    &= P^\pi(s, \mathrm{d}x).
\end{align*}
A symmetric argument shows the second marginal is $P^\pi(s', \cdot)$. Thus, $\mathcal{K} \in \Gamma(P^\pi(s, \cdot), P^\pi(s', \cdot))$.

\paragraph{Cost Evaluation.}
The transport cost associated with the candidate coupling $\mathcal{K}$ is:
\begin{align*}
   \int_{\mathcal{S} \times \mathcal{S}}  d_{\Sstate} (x,y)^2  \, \mathrm{d}\mathcal{K}(x, y)
    &= \int_{\mathcal{S} \times \mathcal{S}}  d_{\Sstate} (x,y)^2  \left( \int_{\mathcal{A} \times \mathcal{A}} \mathrm{d}\Pi_{(s,a), (s',a')}(x, y) \, \mathrm{d}\gamma(a, a') \right)\\
    &= \int_{\mathcal{A} \times \mathcal{A}} \left( \int_{\mathcal{S} \times \mathcal{S}} d_{\Sstate} (x,y)^2  \, \mathrm{d}\Pi_{(s,a), (s',a')}(x, y) \right) \mathrm{d}\gamma(a, a') \\
    &= \int_{\mathcal{A} \times \mathcal{A}} W_2^2(\mathcal{T}(s, a, \cdot), \mathcal{T}(s', a', \cdot)) \, \mathrm{d}\gamma(a, a').
\end{align*}

\paragraph{Conclusion.}
The Wasserstein distance $W_2^2(\mu, \nu)$ is defined as the infimum of the cost over \textit{all} possible couplings. Since $\mathcal{K}$ is one such admissible coupling, we have:
\[
    W_2^2(P^\pi(s, \cdot), P^\pi(s', \cdot)) \leq \int_{\mathcal{S} \times \mathcal{S}}  d_{\Sstate} (x,y)^2  \, \mathrm{d}\mathcal{K}(x, y).
\]
Substituting the cost calculated above, and noting that the inequality holds for any arbitrary $\gamma$, we take the infimum over $\gamma$ to obtain the tightest bound:
\[
    W_2^2(P^\pi(s, \cdot), P^\pi(s', \cdot)) \leq \inf_{\gamma \in \Gamma(\pi(s), \pi(s'))} \int_{\mathcal{A} \times \mathcal{A}} W_2^2(\mathcal{T}(s, a, \cdot), \mathcal{T}(s', a', \cdot)) \, \mathrm{d}\gamma(a, a').
\]
\end{proof}

    Let $\gamma_{s,s'}$ be the optimal coupling between the action distributions $\pi(s)$ and $\pi(s')$. Substituting the Lipschitz condition of the environment $\mathcal{T}$:
    \[
        W_2(\mathcal{T}(s, a, \cdot), \mathcal{T}(s', a', \cdot)) \leq L_\Sstate d_\Sstate(s, s') + L_\Action d_\Action(a, a').
    \]
    Injecting this into the integral and applying the Minkowski inequality (triangle inequality for the $L^2(\gamma_{s,s'})$ norm):
    \begin{align*}
        W_2(P^\pi(s), P^\pi(s')) &\leq \left( \int_{\Action \times \Action} (L_\Sstate d_\Sstate(s, s') + L_\Action d_\Action(a, a'))^2 \, \mathrm{d}\gamma_{s,s'}(a, a') \right)^{1/2} \\
        &\leq \left( \int L_\Sstate^2 d_\Sstate(s, s')^2 \mathrm{d}\gamma_{s,s'} \right)^{1/2} + \left( \int L_\Action^2 d_\Action(a, a')^2 \, \mathrm{d}\gamma_{s,s'} \right)^{1/2} \\
        &= L_\Sstate d_\Sstate(s, s') + L_\Action W_2(\pi(s), \pi(s')).
    \end{align*}

    \textbf{Step 2: Incorporating Policy Lipschitz continuity.}
    Using the assumption that the policy $\pi$ is $K_\pi$-Lipschitz ($W_2(\pi(s), \pi(s')) \leq K_\pi d_\Sstate(s, s')$):
    \[
        W_2(P^\pi(s), P^\pi(s')) \leq (L_\Sstate + L_\Action K_\pi) d_\Sstate(s, s') = \kappa d_\Sstate(s, s').
    \]

\textbf{Step 3: Global Contraction via Coupling Gluing.}
    We now bound the distance between the measures $\nu_1 P^\pi$ and $\nu_2 P^\pi$. To do so, we construct a global candidate coupling by integrating optimal local transition couplings with respect to the optimal initial state coupling $\xi \in \Gamma_o(\nu_1, \nu_2)$.
    
    For every pair of states $(s, s') \in \Sstate \times \Sstate$, let $\Pi_{s,s'}$ be the optimal coupling between the transition probabilities $P^\pi(s, \cdot)$ and $P^\pi(s', \cdot)$. By definition, its cost yields the squared Wasserstein distance:
    \[
        \int_{\Sstate \times \Sstate} d_\Sstate(x, y)^2 \, \mathrm{d}\Pi_{s,s'}(x, y) = W_2^2(P^\pi(s, \cdot), P^\pi(s', \cdot)).
    \]
    
    We define a global candidate coupling $\mathcal{M}$ on $\Sstate \times \Sstate$ as the continuous mixture:
    \[
        \mathcal{M}(\mathrm{d}x, \mathrm{d}y) = \int_{\Sstate \times \Sstate} \Pi_{s,s'}(\mathrm{d}x, \mathrm{d}y) \, \mathrm{d}\xi(s, s'),
    \]
    where the existence of a Borel measurable selection $(s, s') \mapsto \Pi_{s,s'}$ is likewise ensured by \cite[Corollary 5.22]{villani2009optimal}.
    
    First, we verify that $\mathcal{M} \in \Gamma(\nu_1 P^\pi, \nu_2 P^\pi)$. For any measurable set $A \subset \Sstate$, the first marginal is:
    \begin{align*}
        \mathcal{M}(A \times \Sstate) &= \int_{\Sstate \times \Sstate} \Pi_{s,s'}(A \times \Sstate) \, \mathrm{d}\xi(s, s') \\
        &= \int_{\Sstate \times \Sstate} P^\pi(s, A) \, \mathrm{d}\xi(s, s') \quad (\text{since the first marginal of } \Pi_{s,s'} \text{ is } P^\pi(s, \cdot)) \\
&= \int_{\Sstate} P^\pi(s, A) \, \mathrm{d}\nu_1(s) = (\nu_1 P^\pi)(A).
	\end{align*}   
    By symmetry, the second marginal is $\nu_2 P^\pi(A)$. Thus, $\mathcal{M}$ is a valid coupling. Using Fubini's theorem to exchange the integrals:
    \begin{align*}
        W_2^2(\nu_1 P^\pi, \nu_2 P^\pi) &\leq \int_{\Sstate \times \Sstate} d_\Sstate(x, y)^2 \, \mathrm{d}\mathcal{M}(x, y) \\
        &= \int_{\Sstate \times \Sstate} \left[ \int_{\Sstate \times \Sstate} d_\Sstate(x, y)^2 \, \mathrm{d}\Pi_{s,s'}(x, y) \right] \mathrm{d}\xi(s, s') \\
        &= \int_{\Sstate \times \Sstate} W_2^2(P^\pi(s, \cdot), P^\pi(s', \cdot)) \, \mathrm{d}\xi(s, s').
    \end{align*}

    Substituting the pointwise bound established in Step 2 ($W_2(P^\pi(s), P^\pi(s')) \leq \kappa d_\Sstate(s, s')$):
    \begin{align*}
        W_2^2(\nu_1 P^\pi, \nu_2 P^\pi) &\leq \int_{\Sstate \times \Sstate} \kappa^2 d_\Sstate(s, s')^2 \, \mathrm{d}\xi(s, s') \\
        &= \kappa^2 W_2^2(\nu_1, \nu_2).
    \end{align*}
    Taking the square root yields $W_2(\nu_1 P^\pi, \nu_2 P^\pi) \leq \kappa W_2(\nu_1, \nu_2)$.
    
    Since $\kappa < 1$, the operator $\Phi_\pi$ is a strict contraction on the complete metric space $(\mathcal{P}_2(\Sstate), W_2)$. By the Banach Fixed Point Theorem, there exists a unique invariant measure $\mu_\pi$, and for any initial distribution $\mu_0$, the sequence of distributions converges exponentially fast to $\mu_\pi$.
\newline

\end{proof}

Here, we provide another condition that does not require the contraction property but relies on some regularity of the environment and the policy. When the contraction property does not hold, the existence of an invariant measure can alternatively be guaranteed through topological arguments. Relying on the Weak Feller property and the compactness of the state space, this approach is a classical result in the theory of Markov chains, closely related to the Krylov-Bogoliubov theorem, as extensively detailed by \cite{meyn2009markov,hernandez2003markov}. In Theorem \ref{th:existence_invariant_compactness}, we recall this fundamental result and its fixed-point proof for completeness.

\begin{theorem} \cite[Theorem 12.0.1]{meyn2009markov}\label{th:existence_invariant_compactness}
Let the state space $\Sstate$ be a compact metric space. Assume that the environment dynamics and the policy are continuous such that the induced transition kernel $P^\pi$ satisfies the \textbf{Weak Feller property} (i.e., $s \mapsto \int_\Sstate P^\pi(s, ds') f(s')$ is continuous for every $f\in C_b(\Sstate)$).

Under these assumptions, there exists at least one invariant probability measure $\mu_\pi \in \mathcal{P}(\Sstate)$ such that:
\[
    \mu_\pi P^\pi = \mu_\pi.
\]
\end{theorem}

\begin{remark}
Although Theorem \ref{th:existence_invariant_compactness} guarantees the existence of an invariant measure in $\mathcal{P}(\Sstate)$, the compactness of $\Sstate$ implies $\operatorname{diam}(\Sstate) < \infty$, so that $\mathcal{P}(\Sstate) = \mathcal{P}_2(\Sstate)$. Therefore, $\mu_\pi$ inherently possesses finite second moments, fully consistent with our Wasserstein setting.
\end{remark}

\begin{proof}
We define the operator $\Phi_\pi: \mathcal{P}(\Sstate) \to \mathcal{P}(\Sstate)$ by $\Phi_\pi(\nu) = \nu P^\pi$. To establish the existence of a fixed point, we apply the Schauder-Tychonoff fixed point theorem. This requires verifying that the domain is a compact convex set and that the operator is continuous.

\paragraph{1. Geometry of the Domain.}
Let $\mathcal{K} = \mathcal{P}(\Sstate)$ be the set of all probability measures on $\Sstate$.
\begin{itemize}
    \item \textbf{Convexity:} The space of probability measures is convex, as any convex combination of probability measures remains a probability measure.
    \item \textbf{Compactness:} Since $\Sstate$ is a compact metric space, by Prokhorov's Theorem, the set $\mathcal{P}(\Sstate)$ is compact with respect to the topology of weak convergence (weak-* topology).
\end{itemize}

\paragraph{2. Continuity of the Operator.}
We must show that $\Phi_\pi$ is continuous with respect to the weak topology. Let $\{\nu_n\}_{n \in \mathbb{N}}$ be a sequence in $\mathcal{P}(\Sstate)$ converging weakly to $\nu$ (denoted $\nu_n \Rightarrow \nu$). We need to show that $\nu_n P^\pi \Rightarrow \nu P^\pi$.

By the definition of weak convergence, it suffices to show that for any continuous bounded test function $f \in C_b(\Sstate)$:
\[
    \lim_{n \to \infty} \int_{\Sstate} f(s') \, \mathrm{d}(\nu_n P^\pi)(s') = \int_{\Sstate} f(s') \, \mathrm{d}(\nu P^\pi)(s').
\]
Using the definition of the transition kernel, we can rewrite the integral as:
\[
    \int_{\Sstate} f(s') \, \mathrm{d}(\nu_n P^\pi)(s') = \int_{\Sstate} \left( \int_{\Sstate} f(s') P^\pi(s, \mathrm{d}s') \right) \mathrm{d}\nu_n(s).
\]
Let us define the function $g(s) := \int_{\Sstate} f(s') P^\pi(s, \mathrm{d}s')$. The term $g(s)$ represents the expected value of $f$ at the next state starting from $s$. The assumption that $P^\pi$ is weak Feller implies that $g$ is a continuous bounded function, i.e., $g \in C_b(\Sstate)$.

Since $\nu_n \Rightarrow \nu$ and $g$ is continuous bounded, the definition of weak convergence implies:
\[
    \int_{\Sstate} g(s) \, \mathrm{d}\nu_n(s) \xrightarrow{n \to \infty} \int_{\Sstate} g(s) \, \mathrm{d}\nu(s).
\]
Substituting $g$ back into the expression, we confirm that $\Phi_\pi(\nu_n) \Rightarrow \Phi_\pi(\nu)$. Thus, $\Phi_\pi$ is continuous.

\paragraph{Conclusion.}
The operator $\Phi_\pi$ is a continuous map from the non-empty, convex, compact set $\mathcal{P}(\Sstate)$ into itself. By the Schauder-Tychonoff fixed point theorem, there exists $\mu_\pi \in \mathcal{P}(\Sstate)$ such that $\Phi_\pi(\mu_\pi) = \mu_\pi$.
\end{proof}

Now that the existence of the invariant measure is proved, we need to show uniqueness to ensure that the geometric structure presented in Section \ref{sec:wass_geom_optim} is well defined. We give three main theorems  \ref{th:uniqueness_minimal_invariant}, \ref{th:uniqueness_topological_irreducibility}, \ref{th:uniqueness_doeblin} available in the literature of Markov chains in metric spaces \cite{hernandez2003markov,koralov2007theory,hairer2006ergodic, meyn2009markov, da1996ergodicity} proving the uniqueness of the invariant measure for well-behaved Markov kernels. For a deeper investigation on the question of uniqueness when the environment Markov kernel does not have good properties, we refer to the references cited above. If the reader is not concerned with this question, we advise skipping this technical part, assuming the uniqueness of the invariant measure, and going directly to Section \ref{sec:wass_geom_optim}. Before diving into this, we need to define the concept of ergodicity and prove a technical lemma.

\paragraph{Definition (Ergodic Measure).}
A measurable set $B \subset \Sstate$ is called \textbf{$\mu$-invariant} with respect to $P^\pi$ if $P^\pi(s, B) = 1$ for $\mu$-almost every $s \in B$. 

Let $\mathcal{M}_{inv} := \{ \mu \in \mathcal{P}(\Sstate) : \mu P^\pi = \mu \}$ denote the set of invariant probability measures on $\Sstate$. An invariant probability measure $\mu \in \mathcal{M}_{inv}$ is said to be \textbf{ergodic} if for every $\mu$-invariant set $B$, we have $\mu(B) = 0$ or $\mu(B) = 1$.

\begin{lemma} \cite[Theorem 5.7]{hairer2006ergodic}\label{lem:extremal_is_ergodic}
The set $\mathcal{M}_{inv}$ is convex. Moreover, a measure $\mu \in \mathcal{M}_{inv}$ is an extreme point of $\mathcal{M}_{inv}$ if and only if $\mu$ is ergodic.
\end{lemma}

\begin{proof}[Proof of Lemma \ref{lem:extremal_is_ergodic}]

First, we prove the convexity of $ \mathcal{M}_{inv}$.
\paragraph{1. Convexity.}
Let $\mu_1, \mu_2 \in \mathcal{M}_{inv}$ be two invariant measures, and let $\alpha \in [0, 1]$. We define their convex combination as $\nu = \alpha \mu_1 + (1 - \alpha) \mu_2$. Since the transition kernel $P^\pi$ acts linearly on probability measures, we have:
$$
    \nu P^\pi = (\alpha \mu_1 + (1 - \alpha) \mu_2) P^\pi = \alpha (\mu_1 P^\pi) + (1 - \alpha) (\mu_2 P^\pi).
$$
Because $\mu_1$ and $\mu_2$ are invariant, $\mu_1 P^\pi = \mu_1$ and $\mu_2 P^\pi = \mu_2$. Substituting these yields:
$$
    \nu P^\pi = \alpha \mu_1 + (1 - \alpha) \mu_2 = \nu.
$$
Since $\nu P^\pi = \nu$, the measure $\nu$ is also in $\mathcal{M}_{inv}$. Therefore, $\mathcal{M}_{inv}$ is a convex set.

\paragraph{2. Extremal implies Ergodic.}
Suppose $\mu$ is an extreme point of $\mathcal{M}_{inv}$. We assume, for the sake of contradiction, that $\mu$ is not ergodic. Then, there exists a $\mu$-invariant set $B$ such that $0 < \mu(B) < 1$. 

We can define two new probability measures by conditioning $\mu$ on $B$ and its complement $B^c$:
$$
    \mu_1(A) = \frac{\mu(A \cap B)}{\mu(B)} \quad \text{and} \quad \mu_2(A) = \frac{\mu(A \cap B^c)}{1 - \mu(B)}
$$
for any measurable set $A$. Clearly, $\mu$ can be expressed as the strict convex combination:
$$
    \mu = \mu(B)\mu_1 + (1 - \mu(B))\mu_2.
$$
We now show that $\mu_1, \mu_2 \in \mathcal{M}_{inv}$. Let $A \subset \Sstate$ be any Borel set. By definition of $\mu_1$, we have:
\begin{equation}\label{eq:mu1_target}
    (\mu_1 P^\pi)(A) = \int_{\Sstate} P^\pi(s, A) \, \mathrm{d}\mu_1(s) = \frac{1}{\mu(B)} \int_B P^\pi(s, A) \, \mathrm{d}\mu(s).
\end{equation}
To evaluate this integral, we establish two confinement properties of the $\mu$-invariant set $B$:

\paragraph{No probability mass leaves $B$.}
Since $B$ is $\mu$-invariant, $P^\pi(s, B) = 1$ for $\mu$-almost every $s \in B$. Decomposing the target set into $A = (A \cap B) \cup (A \cap B^c)$, the probability of transitioning outside $B$ vanishes:
\[
    P^\pi(s, A \cap B^c) \le P^\pi(s, B^c) = 1 - P^\pi(s, B) = 0 \quad \text{for } \mu\text{-a.e. } s \in B.
\]
Consequently, for $\mu$-almost every $s \in B$:
\begin{equation}\label{eq:no_exit}
    P^\pi(s, A) = P^\pi(s, A \cap B).
\end{equation}

\paragraph{No probability mass enters $B$ from $B^c$.}
Next, we decompose the total mass $\mu(B)$ across $\Sstate = B \cup B^c$:
\[
    \mu(B) = \int_{\Sstate} P^\pi(s, B) \, \mathrm{d}\mu(s) = \int_B P^\pi(s, B) \, \mathrm{d}\mu(s) + \int_{B^c} P^\pi(s, B) \, \mathrm{d}\mu(s).
\]
Because $P^\pi(s, B) = 1$ $\mu$-a.e. on $B$, the first integral equals $\mu(B)$. Subtracting $\mu(B)$ from both sides yields:
\[
    \int_{B^c} P^\pi(s, B) \, \mathrm{d}\mu(s) = 0.
\]
Since $P^\pi \ge 0$, this implies $P^\pi(s, B) = 0$ for $\mu$-almost every $s \in B^c$. Because $A \cap B \subset B$, we also have:
\begin{equation}\label{eq:no_entry}
    P^\pi(s, A \cap B) = 0 \quad \text{for } \mu\text{-a.e. } s \in B^c.
\end{equation}

\paragraph{Verification of invariance.}
We now evaluate the integral over $B$ in \eqref{eq:mu1_target}:
\begin{align*}
    \int_B P^\pi(s, A) \, \mathrm{d}\mu(s) 
    &= \int_B P^\pi(s, A \cap B) \, \mathrm{d}\mu(s) && \text{(by \eqref{eq:no_exit})} \\
    &= \int_B P^\pi(s, A \cap B) \, \mathrm{d}\mu(s) + \int_{B^c} P^\pi(s, A \cap B) \, \mathrm{d}\mu(s)  \\
    &= \int_{\Sstate} P^\pi(s, A \cap B) \, \mathrm{d}\mu(s) && \text{(by \eqref{eq:no_entry})} \\
    &= \mu(A \cap B) && \text{(by invariance of } \mu \text{)}.
\end{align*}
Dividing by $\mu(B)$, we obtain:
\[
    (\mu_1 P^\pi)(A) = \frac{\mu(A \cap B)}{\mu(B)} = \mu_1(A),
\]
which proves that $\mu_1 \in \mathcal{M}_{inv}$. A symmetric argument shows that $\mu_2 \in \mathcal{M}_{inv}$.

Since $\mu_1(B) = 1$ and $\mu_2(B) = 0$, $\mu_1$ and $\mu_2$ are distinct. We have expressed $\mu$ as a non-trivial convex combination of two distinct invariant measures, contradicting the assumption that $\mu$ is an extreme point. Thus, $\mu$ must be ergodic.

\paragraph{3: Ergodic implies Extremal.} 
Suppose $\mu$ is ergodic, and assume $\mu = \alpha \mu_1 + (1 - \alpha) \mu_2$ for some $\alpha \in (0, 1)$ and $\mu_1, \mu_2 \in \mathcal{M}_{inv}$. 

Because $\alpha > 0$ and the measures are non-negative, $\mu_1$ is absolutely continuous with respect to $\mu$ ($\mu_1 \ll \mu$). By the Radon-Nikodym theorem, there exists a measurable density function $f = \frac{\mathrm{d}\mu_1}{\mathrm{d}\mu}$ such that $0 \le f \le \frac{1}{\alpha}$ $\mu$-a.e.

For any real constant $c$, we claim that the super-level set $B_c = \{s \in \Sstate : f(s) > c\}$ is a $\mu$-invariant set (the claim is proved at the end of the proof). Because $\mu$ is ergodic, $\mu(B_c)$ must be exactly $0$ or $1$ for all $c$. As this is true for all real $c$, this forces the density $f$ to be constant $\mu$-almost everywhere. Since $f$ is a probability density, $\int f \, \mathrm{d}\mu = 1$, which means $f(s) = 1$ $\mu$-a.e. 

Therefore, $\mu_1 = \mu$, which immediately forces $\mu_2 = \mu$. Since $\mu$ cannot be decomposed into a non-trivial convex combination of distinct invariant measures, it is an extreme point. 
\begin{proof}[Proof of the Claim]
By definition of the Radon-Nikodym derivative $f = \frac{\mathrm{d}\mu_1}{\mathrm{d}\mu}$, the probability assigned to any measurable set $A$ by $\mu_1$ is:
\[
    \mu_1(A) = \int_A f(s) \, \mathrm{d}\mu(s).
\]
On the other hand, since $\mu_1$ is an invariant measure ($\mu_1 = \mu_1 P^\pi$), we can express the action of the transition kernel on $\mu_1$ by integrating against $f \, \mathrm{d}\mu$:
\[
    \mu_1(A) = (\mu_1 P^\pi)(A) = \int_{\Sstate} P^\pi(s, A) \, \mathrm{d}\mu_1(s) = \int_{\Sstate} P^\pi(s, A) f(s) \, \mathrm{d}\mu(s).
\]
Equating these two expressions for $\mu_1(A)$ yields, for every measurable set $A$:
\begin{equation}\label{eq:mu_invariant_f_A}
    \int_A f(s) \, \mathrm{d}\mu(s) = \int_{\Sstate} P^\pi(s, A) f(s) \, \mathrm{d}\mu(s).
\end{equation}

Substituting $A = B_c$ into the equality above yields:
$$
    \int_{B_c} f(s) \, \mathrm{d}\mu(s) = \int_{B_c} P^\pi(s, B_c) f(s) \, \mathrm{d}\mu(s) + \int_{B_c^c} P^\pi(s, B_c) f(s) \, \mathrm{d}\mu(s).
$$
Because $\mu$ is an invariant measure, we also have $\mu(B_c) = \int_{\Sstate} P^\pi(s, B_c) \, \mathrm{d}\mu(s)$. Multiplying this by $c$ gives:
$$
    \int_{B_c} c \, \mathrm{d}\mu(s) = \int_{B_c} P^\pi(s, B_c) c \, \mathrm{d}\mu(s) + \int_{B_c^c} P^\pi(s, B_c) c \, \mathrm{d}\mu(s).
$$
Subtracting this equation from \eqref{eq:mu_invariant_f_A}, we obtain:
$$
    \int_{B_c} (f(s) - c) \, \mathrm{d}\mu(s) = \int_{B_c} P^\pi(s, B_c) (f(s) - c) \, \mathrm{d}\mu(s) + \int_{B_c^c} P^\pi(s, B_c) (f(s) - c) \, \mathrm{d}\mu(s).
$$
Rearranging the terms by bringing the integral over $B_c$ to the left side:
$$
    \int_{B_c} (1 - P^\pi(s, B_c)) (f(s) - c) \, \mathrm{d}\mu(s) = \int_{B_c^c} P^\pi(s, B_c) (f(s) - c) \, \mathrm{d}\mu(s).
$$
Now, we analyze the signs of both sides. On the left side, for $s \in B_c$, we have $f(s) - c > 0$ by definition, and $1 - P^\pi(s, B_c) \ge 0$ because $P^\pi$ is a probability kernel. Thus, the left-hand integrand is non-negative, meaning the left side is $\ge 0$. On the right side, for $s \in B_c^c$, we have $f(s) - c \le 0$, and $P^\pi(s, B_c) \ge 0$. Thus, the right-hand integrand is non-positive, meaning the right side is $\le 0$.

For a non-negative quantity to equal a non-positive quantity, both must be exactly zero. Therefore, the left side is zero:
$$
    \int_{B_c} (1 - P^\pi(s, B_c)) (f(s) - c) \, \mathrm{d}\mu(s) = 0.
$$
Since the integrand $(1 - P^\pi(s, B_c)) (f(s) - c)$ is strictly positive whenever $P^\pi(s, B_c) < 1$, the integral can only be zero if $1 - P^\pi(s, B_c) = 0$ for $\mu$-almost every $s \in B_c$. This implies $P^\pi(s, B_c) = 1$ $\mu$-a.e. on $B_c$, which is exactly the definition of $B_c$ being a $\mu$-invariant set.

\end{proof}

This finishes the proof of the Lemma.
\end{proof}

\begin{lemma}[Mutual Singularity of Ergodic Measures]\label{lem:ergodic_mutually_singular}
Let $\mu_1, \mu_2 \in \mathcal{M}_{inv}$ be two distinct ergodic measures ($\mu_1 \neq \mu_2$). Then, $\mu_1$ and $\mu_2$ are mutually singular ($\mu_1 \perp \mu_2$), meaning there exists a Borel set $A \subset \Sstate$ such that $\mu_1(A) = 1$ and $\mu_2(A) = 0$.
\end{lemma}

\begin{proof}

Consider the Lebesgue decomposition of $\mu_1$ with respect to $\mu_2$:
\[
    \mu_1 = \mu_{1,\mathrm{ac}} + \mu_{1,\mathrm{s}},
\]
where $\mu_{1,\mathrm{ac}} \ll \mu_2$ and $\mu_{1,\mathrm{s}} \perp \mu_2$.

\paragraph*{Measures $\mu_{1,\mathrm{ac}}$ and $\mu_{1,\mathrm{s}}$ are both invariant}

First, we verify that $\mu_{1,\mathrm{ac}} P^\pi \ll \mu_2$. Let $B \in \Borel(\Sstate)$ such that $\mu_2(B) = 0$. By invariance of $\mu_2$, $\int_\Sstate P^\pi(s, B) \, \mathrm{d}\mu_2(s) = 0$, which implies $P^\pi(s, B) = 0$ for $\mu_2$-almost every $s$. Since $\mu_{1,\mathrm{ac}} \ll \mu_2$, this also holds $\mu_{1,\mathrm{ac}}$-almost everywhere, giving:
\[
    (\mu_{1,\mathrm{ac}} P^\pi)(B) = \int_\Sstate P^\pi(s, B) \, \mathrm{d}\mu_{1,\mathrm{ac}}(s) = 0,
\]
proving that $\mu_{1,\mathrm{ac}} P^\pi \ll \mu_2$.

Next, applying the invariant kernel $P^\pi$ to $\mu_1$ yields $\mu_1 = \mu_1 P^\pi = \mu_{1,\mathrm{ac}} P^\pi + \mu_{1,\mathrm{s}} P^\pi$. Decomposing $\mu_{1,\mathrm{s}} P^\pi = (\mu_{1,\mathrm{s}} P^\pi)_{\mathrm{ac}} + (\mu_{1,\mathrm{s}} P^\pi)_{\mathrm{s}}$ with respect to $\mu_2$, the uniqueness of the Lebesgue decomposition of $\mu_1$ forces:
\[
    \mu_{1,\mathrm{ac}} = \mu_{1,\mathrm{ac}} P^\pi + (\mu_{1,\mathrm{s}} P^\pi)_{\mathrm{ac}}.
\]
Because $P^\pi$ is a Markov kernel, it preserves total mass: $(\mu_{1,\mathrm{ac}} P^\pi)(\Sstate) = \mu_{1,\mathrm{ac}}(\Sstate)$. Evaluating the identity on $\Sstate$ gives $(\mu_{1,\mathrm{s}} P^\pi)_{\mathrm{ac}}(\Sstate) = 0$, which forces $(\mu_{1,\mathrm{s}} P^\pi)_{\mathrm{ac}} = 0$. 

Therefore, each component is individually invariant:
\[
    \mu_{1,\mathrm{ac}} P^\pi = \mu_{1,\mathrm{ac}} \quad \text{and} \quad \mu_{1,\mathrm{s}} P^\pi = \mu_{1,\mathrm{s}}.
\]

\paragraph*{Proof by contradiction}
Next, suppose for the sake of contradiction, that $\mu_{1,\mathrm{ac}} \neq 0$. Setting $\alpha := \mu_{1,\mathrm{ac}}(\Sstate) \in (0, 1]$, we define the normalized invariant probability measure:
\[
    \nu := \frac{\mu_{1,\mathrm{ac}}}{\alpha} \in \mathcal{M}_{inv}.
\]
By construction, $\nu$ remains absolutely continuous with respect to $\mu_2$ ($\nu \ll \mu_2$).

By the Radon-Nikodym theorem, there exists a density $g = \frac{\mathrm{d}\nu}{\mathrm{d}\mu_2} \ge 0$. As established in the Claim of Lemma \ref{lem:extremal_is_ergodic}, the super-level sets $B_c := \{s \in \Sstate \mid g(s) > c\}$ are $\mu_2$-invariant for every $c \in \mathbb{R}$.

Because $\mu_2$ is ergodic, every $\mu_2$-invariant set must have measure $0$ or $1$. Consequently, $g$ must be constant $\mu_2$-almost everywhere. Since $\int_\Sstate g \, \mathrm{d}\mu_2 = \nu(\Sstate) = 1$, we have $g = 1$ $\mu_2$-a.e., which implies:
\[
    \nu = \mu_2, \quad \text{and thus} \quad \mu_{1,\mathrm{ac}} = \alpha \mu_2.
\]

Substituting this back yields the decomposition:
\[
    \mu_1 = \alpha \mu_2 + \mu_{1,\mathrm{s}}.
\]
We analyze the value of $\alpha \in (0, 1]$:
\begin{itemize}
    \item If $\alpha = 1$, then $\mu_{1,\mathrm{s}}(\Sstate) = 0$, giving $\mu_1 = \mu_2$, which directly contradicts the assumption $\mu_1 \neq \mu_2$.
    \item If $\alpha \in (0, 1)$, we define the invariant probability measure $\tilde{\mu} := \frac{\mu_{1,\mathrm{s}}}{1 - \alpha} \in \mathcal{M}_{inv}$ (by the first part of the proof), yielding the strict convex combination:
    \[
        \mu_1 = \alpha \mu_2 + (1 - \alpha)\tilde{\mu}.
    \]
    However, by Lemma \ref{lem:extremal_is_ergodic}, $\mu_1$ is an extreme point of $\mathcal{M}_{inv}$, which forbids any convex representation involving a distinct measure $\mu_2 \neq \mu_1$.
\end{itemize}

In both cases, we reach a contradiction. Therefore, $\mu_{1,\mathrm{ac}} = 0$, which proves that $\mu_1 = \mu_{1,\mathrm{s}}$, i.e., $\mu_1 \perp \mu_2$.
\end{proof}

\begin{definition}
A closed set $C \subset \Sstate$ is called \textbf{invariant} with respect to the transition kernel $P^\pi$ if for all $s \in C$, we have $P^\pi(s, C) = 1$.
\end{definition}
 This is a stronger condition than $\mu$-invariance, as it requires the property to hold for \textit{every} state in the set, not just almost everywhere. Now, we prove that the support of an invariant measure is itself invariant.

\begin{lemma}\label{lemma:supp_invariant}
Let $\mu \in \mathcal{M}_{inv}$; then $\text{supp}(\mu)$ is invariant with respect to the Weak Feller transition kernel $P^\pi$.
\end{lemma} 

\begin{proof}
Let $C = \text{supp}(\mu)$. By definition, $C$ is a closed set, and its complement $U = C^c = \Sstate \setminus C$ is an open set with $\mu$-measure zero, i.e., $\mu(U) = 0$.

Using the invariance of $\mu$, we have:
\[
    \mu(U) = \int_{\Sstate} P^\pi(s, U) \, \mathrm{d}\mu(s).
\]
Since $\mu(U) = 0$, it follows that:
\begin{equation}\label{eq:Ppi_U_zero}
    \int_{\Sstate} P^\pi(s, U) \, \mathrm{d}\mu(s) = 0.
\end{equation}
Therefore, $P^\pi(s, U) = 0$ for $\mu$-almost every $s$. 

Since $U$ is an open set in a metric space, there exists a sequence of continuous bounded functions $0 \le f_n \le 1$ such that $f_n(s) \nearrow 1_U(s)$ pointwise (e.g., using the distance function to the closed set $C$). By the Monotone Convergence Theorem:
\[
    P^\pi(s, U) = \int_{\Sstate} 1_U(s') P^\pi(s, \mathrm{d}s') = \lim_{n \to \infty} \int_{\Sstate} f_n(s') P^\pi(s, \mathrm{d}s').
\]
Let $g_n(s) = \int_{\Sstate} f_n(s') P^\pi(s, \mathrm{d}s')$. By the Weak Feller assumption, each $g_n$ is continuous. The function $\phi(s) := P^\pi(s, U)$ is the pointwise limit of an increasing sequence of continuous functions, which implies that $\phi$ is lower semi-continuous (LSC).

We proceed by contradiction. Assume there exists a point $s_0 \in C$ such that $P^\pi(s_0, C) < 1$. Equivalently, $\phi(s_0) = P^\pi(s_0, U) > 0$.
Since $\phi$ is lower semi-continuous, there exists an open neighborhood $V$ of $s_0$ and some $\epsilon > 0$ such that $\phi(s) > \epsilon$ for all $s \in V$.

Because $s_0$ lies in the support of $\mu$, any open neighborhood of $s_0$ must have positive measure. Thus, $\mu(V) > 0$.
We can now bound the integral:
\[
    \int_{\Sstate} P^\pi(s, U) \, \mathrm{d}\mu(s) \ge \int_V \phi(s) \, \mathrm{d}\mu(s) \ge \epsilon \cdot \mu(V) > 0.
\]
This contradicts \eqref{eq:Ppi_U_zero}. Therefore, we must have $P^\pi(s, U) = 0$ for all $s \in C$.
Consequently, $P^\pi(s, C) = 1$ for all $s \in C$, proving that $\text{supp}(\mu)$ is an invariant set.
\end{proof}

Now we are able to state and prove the main abstract uniqueness theorem; we just need to introduce the \textbf{Strong Feller} property. The kernel $P^\pi$ is said to be Strong Feller if for every bounded Borel measurable function $f \in B_b(\Sstate)$, the function $s \mapsto \int_\Sstate f(s') \, P^\pi(s, \mathrm{d}s')$ is continuous. Obviously, a Strong Feller kernel is Weak Feller.
\begin{theorem}[Unique Minimal Closed Invariant Set, Da Prato and Zabczyk \cite{da1996ergodicity}]\label{th:uniqueness_minimal_invariant}
Let the state space $\Sstate$ be a compact metric space and assume $P^\pi$ is Strong Feller. If there exists exactly one minimal closed invariant set in $\Sstate$, then the invariant probability measure $\mu_\pi$ is unique.
\end{theorem}

\begin{proof}
Suppose, for the sake of contradiction, that there exist two distinct invariant probability measures. First, we claim that $\mathcal{M}_{inv}$, the set of invariant measures, on a compact space is convex and compact.  

\begin{proof}[Proof of the Claim]
The convexity of $\mathcal{M}_{inv}$ is given by Lemma \ref{lem:extremal_is_ergodic}. Because the state space $\Sstate$ is a compact metric space, the space of all probability measures $\mathcal{P}(\Sstate)$ is compact with respect to the weak topology. 

Recall the operator $\Phi_\pi: \mathcal{P}(\Sstate) \to \mathcal{P}(\Sstate)$ defined by $\Phi_\pi(\mu) = \mu P^\pi$. Since $P^\pi$ is Strong Feller, it is trivially also Weak Feller. As established in the existence proof, this guarantees that $\Phi_\pi$ is continuous with respect to the weak topology.

The set of invariant measures $\mathcal{M}_{inv}$ is precisely the set of fixed points of $\Phi_\pi$:
$$
    \mathcal{M}_{inv} = \{ \mu \in \mathcal{P}(\Sstate) : \Phi_\pi(\mu) = \mu \}.
$$
In any Hausdorff topological space (which $\mathcal{P}(\Sstate)$ is under the weak topology), the set of fixed points of a continuous mapping is closed as it is the kernel of a continuous operator. Thus, $\mathcal{M}_{inv}$ is a closed subset of the compact space $\mathcal{P}(\Sstate)$. Since any closed subset of a compact space is itself compact, $\mathcal{M}_{inv}$ is compact.
\end{proof}

By the claim and the Krein-Milman Theorem, $\mathcal{M}_{inv}$ is equal to the closed convex hull of its extreme points. Thus, as the existence of multiple invariant measures is assumed, there exist at least two distinct extreme points $\mu_1 \neq \mu_2$. By Lemma \ref{lem:extremal_is_ergodic}, $\mu_1$ and $\mu_2$ are ergodic.

Because $\mu_1$ and $\mu_2$ are distinct ergodic measures, they are mutually singular ($\mu_1 \perp \mu_2$) by Lemma \ref{lem:ergodic_mutually_singular}. Therefore, there exists a Borel measurable set $A \subset \Sstate$ such that $\mu_1(A) = 1$ and $\mu_2(A) = 0$. 
We now demonstrate that their topological supports are disjoint by leveraging the Strong Feller property. By the invariance of $\mu_1$, we have:
$$
    \int_{\Sstate} P^\pi(s, A) \, \mathrm{d}\mu_1(s) = \mu_1(A) = 1
$$
Since $0 \leq P^\pi(s, A) \leq 1$, it must be that $P^\pi(s, A) = 1$ for $\mu_1$-almost every $s$. Because $P^\pi$ is Strong Feller and the indicator function $\mathbf{1}_A$ is a bounded measurable function, the mapping $\varphi(s) := (P^\pi \mathbf{1}_A)(s) = \int_{\Sstate} \mathbf{1}_A(s') P^\pi(s, \mathrm{d}s') = P^\pi(s, A)$ is continuous. Thus, the pre-image $F_1 = \varphi^{-1}(\{1\})$ is a closed set. Since the condition $P^\pi(s, A) = 1$ holds for $\mu_1$-almost every $s$, the set $F_1$ (which collects all points satisfying this equality) necessarily has full measure, i.e., $\mu_1(F_1) = 1$. By definition, the topological support of a measure is the intersection of all closed sets of measure 1, which implies $\operatorname{supp}(\mu_1) \subseteq F_1$.

Symmetrically, by the invariance of $\mu_2$:
$$
    \int_{\Sstate} P^\pi(s, A) \, \mathrm{d}\mu_2(s) = \mu_2(A) = 0
$$
This implies $P^\pi(s, A) = 0$ for $\mu_2$-almost every $s$. The continuous pre-image $F_2 = \varphi^{-1}(\{0\})$ is a closed set with $\mu_2(F_2) = 1$, so $\operatorname{supp}(\mu_2) \subseteq F_2$.

Since $F_1 \cap F_2 = \emptyset$, it strictly follows that $\text{supp}(\mu_1) \cap \text{supp}(\mu_2) = \emptyset$.

Because $\Sstate$ is a compact metric space, any closed invariant set contains at least one minimal closed invariant set (which can be shown via Zorn's Lemma). Therefore, as $\text{supp}(\mu_1)$ is invariant by Lemma \ref{lemma:supp_invariant}, it contains a minimal closed invariant set $M_1$. Symmetrically, $\text{supp}(\mu_2)$ contains a minimal closed invariant set $M_2$.

Since $\text{supp}(\mu_1) \cap \text{supp}(\mu_2) = \emptyset$, it follows that $M_1 \cap M_2 = \emptyset$, meaning $M_1$ and $M_2$ are distinct minimal closed invariant sets. This contradicts the assumption that $\Sstate$ contains exactly one minimal closed invariant set. Hence, the invariant measure $\mu_\pi$ must be unique.
\end{proof}

\vspace{1em}

The unique closed invariant subset criterion is rather abstract and difficult to prove. The topological irreducibility criterion gives another perspective.
\begin{theorem} \label{th:uniqueness_topological_irreducibility}
\textbf{(Doob's theorem, \cite[Chapter 4]{da1996ergodicity})}
Let the state space $\Sstate$ be a compact metric space and assume $P^\pi$ is Strong Feller. Assume $P^\pi$ is topologically irreducible, meaning for every $s \in \Sstate$ and every non-empty open set $O \subset \Sstate$, there exists an integer $n \ge 1$ such that:
\[
    (P^\pi)^n(s, O) > 0
\]
Then there exists a unique invariant probability measure $\mu_\pi$.
\end{theorem}

\begin{proof}
We show that topological irreducibility implies that $\Sstate$ is the only non-empty closed invariant set, which guarantees the existence of a unique minimal closed invariant set.

Clearly, $\Sstate$ itself is closed and invariant since $P^\pi(s, \Sstate) = 1$ for all $s \in \Sstate$. Now, let $C \subseteq \Sstate$ be any non-empty closed invariant set. 

If $C$ were a proper subset of $\Sstate$, its complement $O = \Sstate \setminus C$ would be a non-empty open set. By invariance of $C$, we would have $(P^\pi)^n(s, O) = 0$ for all $s \in C$ and all $n \ge 1$, directly contradicting topological irreducibility.

Thus, no non-empty proper closed invariant subset exists, meaning $\Sstate$ is the unique non-empty closed invariant set. Consequently, $\Sstate$ is trivially the unique minimal closed invariant set, and uniqueness of $\mu_\pi$ follows from Theorem \ref{th:uniqueness_minimal_invariant}.
\end{proof}

\vspace{1em}

Finally, we present Doeblin's criterion, which corresponds to a lower bound of the kernel by a probability measure independent of the state variable.
\begin{theorem}[Doeblin's Criterion, Hairer \cite{hairer2006ergodic}]\label{th:uniqueness_doeblin}
Assume there exists a probability measure $\eta \in \mathcal{P}(\Sstate)$, an integer $m \ge 1$, and a constant $\alpha \in (0, 1)$ such that for all states $s \in \Sstate$ and all Borel sets $B$:
\[
    (P^\pi)^m(s, B) \ge \alpha \eta(B).
\]
Then the invariant probability measure $\mu_\pi$ is unique.
\end{theorem}

\begin{proof}
Let $\mu_1$ and $\mu_2$ be two invariant probability measures for the transition kernel $P^\pi$. By definition of invariance, they are also invariant for the $m$-step transition kernel, meaning $\mu_1 (P^\pi)^m = \mu_1$ and $\mu_2 (P^\pi)^m = \mu_2$.

Recall that the total variation (TV) distance between two probability measures $\nu_1$ and $\nu_2$ on a measurable space $(\Sstate, \mathcal{B})$ is defined as the maximum difference in probability they assign to any single event:
\[
    \|\nu_1 - \nu_2\|_{TV} = \sup_{B \in \mathcal{B}} |\nu_1(B) - \nu_2(B)|.
\]

We evaluate the total variation distance between the invariant measures $\mu_1$ and $\mu_2$. By the minorization condition, we can decompose the $m$-step transition kernel as:
\[
    (P^\pi)^m(s, \cdot) = \alpha \eta(\cdot) + (1 - \alpha) R(s, \cdot)
\]
where $R(s, \cdot) := \frac{1}{1-\alpha}( (P^\pi)^m(s, \cdot) - \alpha \eta(\cdot)) $ is a valid residual transition kernel for each $s \in \Sstate$. (It is a valid kernel because $(P^\pi)^m(s, B) \ge \alpha \eta(B)$ ensures non-negativity, and $R(s, \Sstate) = \frac{1 - \alpha}{1 - \alpha} = 1$).

For any Borel set $B \in \mathcal{B}$, we apply the $m$-step transition kernel to the measure $\mu_1$:
\[
    (\mu_1 (P^\pi)^m)(B) = \int_{\Sstate} (P^\pi)^m(s, B) \, \mathrm{d}\mu_1(s) = \int_{\Sstate} \left( \alpha \eta(B) + (1 - \alpha) R(s, B) \right) \mathrm{d}\mu_1(s).
\]
Since $\mu_1$ is a probability measure, $\int_{\Sstate} \mathrm{d}\mu_1(s) = 1$. This allows us to pull the constant terms out:
\[
    (\mu_1 (P^\pi)^m)(B) = \alpha \eta(B) + (1 - \alpha) \int_{\Sstate} R(s, B) \, \mathrm{d}\mu_1(s) = \alpha \eta(B) + (1 - \alpha) (\mu_1 R)(B).
\]
By identical logic for $\mu_2$, we have:
\[
    (\mu_2 (P^\pi)^m)(B) = \alpha \eta(B) + (1 - \alpha) (\mu_2 R)(B).
\]

Now, we examine the difference between the two measures. Because they are invariant under $(P^\pi)^m$, $\mu_1(B) - \mu_2(B) = (\mu_1 (P^\pi)^m)(B) - (\mu_2 (P^\pi)^m)(B)$. Substituting our expanded forms, the common overlapping component $\alpha \eta(B)$ exactly cancels out:
\[
    \mu_1(B) - \mu_2(B) = (1 - \alpha) \left( (\mu_1 R)(B) - (\mu_2 R)(B) \right).
\]

Taking the supremum over all Borel sets $B$ to find the total variation distance yields:
\[
    \|\mu_1 - \mu_2\|_{TV} = \sup_{B \in \mathcal{B}} |\mu_1(B) - \mu_2(B)| = (1 - \alpha) \sup_{B \in \mathcal{B}} |(\mu_1 R)(B) - (\mu_2 R)(B)|.
\]
This is equivalent to:
\[
    \|\mu_1 - \mu_2\|_{TV} = (1 - \alpha) \|\mu_1 R - \mu_2 R\|_{TV}.
\]

A fundamental property of any Markov transition kernel (including our residual kernel $R$) is that it acts as a weak contraction; applying it to two measures cannot increase the total variation distance between them (i.e., $\|\mu_1 R - \mu_2 R\|_{TV} \le \|\mu_1 - \mu_2\|_{TV}$). 

Applying this bound, we obtain:
\[
    \|\mu_1 - \mu_2\|_{TV} \le (1 - \alpha) \|\mu_1 - \mu_2\|_{TV}.
\]

Since the minorization condition states that $\alpha > 0$, we have $1 - \alpha < 1$. The total variation distance is by definition non-negative ($\|\mu_1 - \mu_2\|_{TV} \ge 0$). The inequality $x \le c x$ for $c < 1$ and $x \ge 0$ can only hold true if $x = 0$. 

Thus, $\|\mu_1 - \mu_2\|_{TV} = 0$, which implies that $\mu_1 = \mu_2$, proving that the invariant measure is unique.
\end{proof}

\section{Wasserstein Geometry and Optimization}\label{sec:wass_geom_optim}

We aim to endow the policy space $\Pi$ with a formal Riemannian structure. To achieve this, we analyze the geometry of curves in the policy space using the theory of absolutely continuous curves in the Wasserstein space $\Ptwo(\Action)$.
From now on, the action space is assumed to be $\R^{d_a}$ or a compact convex subset of $\R^{d_a}$, and the state space is assumed to be $\R^{d_s}$ or a compact convex subset of $\R^{d_s}$ with $d_a, d_s > 0$.
\subsection{The Tangent Space and Metric Derivative}

Let $(\pi_t)_{t \in I}$ be a time-dependent policy, where $I \subset \mathbb{R}$ is an open interval. For a fixed state $s \in \Sstate$, the curve $t \mapsto \pi_t(s)$ represents a trajectory in the space of measures.

According to the Ambrosio-Gigli-Savaré \cite{ambrosio2008gradient} theory (specifically Theorem 8.3.1), if this curve is absolutely continuous of order 2 ($AC^2$) with respect to the Wasserstein distance $W_2$, it admits a square-integrable metric derivative, denoted by $|\pi_t'|(s)$, defined for almost every $t$ by:

\[
    |\pi_t'|(s) := \lim_{h \to 0} \frac{W_2(\pi_{t+h}(s), \pi_t(s))}{|h|}.
\]
This scalar quantity measures the minimal amount of transport ``work'' required to move the probability mass at rate $t$. The geometry of the transport is described by a vector field $v_t(s, \cdot)$ satisfying the continuity equation:
\[
    \partial_t \pi_t(s) + \mathrm{div}_a (\pi_t(s) v_t(s)) = 0.
\]
While there are infinitely many vector fields satisfying this equation, there exists a unique field $v_t(s, \cdot)$ belonging to the tangent space (the closure of gradients) that characterizes the optimal transport. This field is uniquely identified by the condition that its $L^2$ norm is bounded by the metric derivative: $\|v_t(s)\|_{L^2(\pi_t(s))} \leq |\pi_t'|(s)$. In fact, equality holds for this optimal choice.

In order to compute integral estimates on $v_t(s,a)$, we need to show that this map is measurable. This is proven in the following theorem.

\begin{theorem}[Global Measurable Selection of Tangent Fields]
Let $I \subset \mathbb{R}$ be an interval and $\pi: I=[0,T] \to \Pi$ be a dynamic policy such that:
\begin{enumerate}
    \item For every $t \in I$, the map $s \mapsto \pi_t(s)$ is measurable from $(\Sstate, \mathcal{B}(\Sstate))$ to $(\Ptwo(\Action), \mathcal{B}(\Ptwo(\Action)))$.
    \item For every $s \in \Sstate$, the curve $t \mapsto \pi_t(s)$ belongs to $AC^2(I; \Ptwo(\Action))$.
\end{enumerate}

Then there exists a map $(t, s, a) \mapsto v_t(s, a)$ defined on $I \times \Sstate \times \Action$, which is jointly measurable with respect to the product $\sigma$-algebra $\mathcal{B}(I) \otimes \mathcal{B}(\Sstate) \otimes \mathcal{B}(\Action)$, such that, for every $s \in \Sstate$, the vector field $v_t(s, \cdot)$ acts as the tangent velocity field for almost every $t \in I$. Moreover:

\begin{enumerate}
\item For every $\phi \in C_c^\infty((0,T) \times \Action)$, the field satisfies the distributional equation state-wise:
    \[
        \int_0^T \int_{\Action} \left( \partial_t \phi(t, a) + \langle v_t(s, a), \nabla_a \phi(t, a) \rangle \right) \mathrm{d}\pi_t(s)(a) \, \mathrm{d}t = 0.
    \]
\item For every $s \in \Sstate$ and almost every $t \in I$, the vector field satisfies the norm constraint:
    \[
        \| v_t(s, \cdot) \|_{L^2(\pi_t(s))} \leq |\pi_t'|(s),
    \]
    where $|\pi_t'|(s)$ is the metric derivative of the curve $t \mapsto \pi_t(s)$.
\end{enumerate}
\end{theorem}

\begin{proof}
The proof relies on identifying the vector field as the unique solution to a convex optimization problem depending measurably on parameters $(t,s)$, and applying a measurable selection theorem.
\paragraph{1. Pointwise Existence and Uniqueness.}
Fix a state $s \in \Sstate$. Since the curve $t \mapsto \pi_t(s)$ belongs to $AC^2(I; \Ptwo(\Action))$, by \cite[Theorem 8.3.1]{ambrosio2008gradient}, the metric derivative $|\pi_t'|(s)$ exists for almost every $t$ and belongs to $L^2(I)$.

The same theorem states that there exists a unique vector field $v_t(s, \cdot) \in \overline{\nabla C_c^\infty}^{L^2(\pi_t(s))}$ satisfying the continuity equation. Furthermore, this unique field is characterized by the minimality condition:
\[
    \| v_t(s, \cdot) \|_{L^2(\pi_t(s))} = |\pi_t'|(s).
\]
Any other vector field $\tilde{v}$ satisfying the continuity equation must have strictly greater norm, i.e., $\|\tilde{v}\| > |\pi_t'|(s)$. Thus, the condition $\|v\| \leq |\pi_t'|(s)$ appearing in the theorem statement effectively selects this unique minimal solution.

\paragraph{2. Measurability of the Metric Derivative.}
We first establish the global measurability of the scalar field representing the metric derivative.
The function $(t, t', s) \mapsto W_2(\pi_t(s), \pi_{t'}(s))$ is measurable because $\pi$ is jointly measurable (as it is measurable in $s$ and continuous in $t$) and $W_2$ is continuous.
We define the metric derivative field globally via the lim sup along a countable sequence:
\[
    |\pi_t'|(s) := \limsup_{n \to \infty} n \, W_2(\pi_{t+1/n}(s), \pi_t(s)).
\]
Since the lim sup of a sequence of measurable functions is always globally measurable, the map $(t, s) \mapsto |\pi_t'|(s)$ is Borel measurable on the entire product space $I \times \Sstate$. Furthermore, by the $2$-absolute continuity of the curve $t \mapsto \pi_t(s)$, for every fixed $s \in \Sstate$, this limit superior coincides with the true metric derivative for Lebesgue-almost every $t \in I$.

\paragraph{3. Definition and Measurability of the Base Measure $\mu^s$.}
Let $X = I \times \Action$. We define the base action-time finite measure $\mu^s \in \mathcal{M}_+(X)$ (the space of finite non-negative Radon measures) by $\mathrm{d}\mu^s(t, a) = \mathrm{d}\pi_t(s)(a) \mathrm{d}t$. Note that its total mass is exactly $T$.

We will establish that the map $s \mapsto \mu^s$ is Borel measurable from $\Sstate$ into $\mathcal{M}_+(X)$ equipped with the weak-* topology. The weak-* topology on $\mathcal{M}_+(X)$ is generated by the evaluation functionals $\nu \mapsto \int_X \phi \, \mathrm{d}\nu$ for any test function $\phi \in C_0(X)$. Because $C_0(X)$ is separable, the Borel $\sigma$-algebra on the measure space is entirely generated by these evaluation maps. Therefore, it suffices to show that for any continuous test function $\phi \in C_0(X)$, the scalar evaluation map $s \mapsto \int_X \phi(t, a) \, \mathrm{d}\mu^s(t, a)$ is Borel measurable. 

By the definition of $\mu^s$, this integral expands to:
\[
    s \mapsto \int_I \left( \int_\Action \phi(t, a) \, \mathrm{d}\pi_t(s)(a) \right) \mathrm{d}t.
\]
Let $g(t, s) := \int_\Action \phi(t, a) \, \mathrm{d}\pi_t(s)(a)$. For a fixed $t$, the map $s \mapsto g(t, s)$ is Borel measurable because the policy map $s \mapsto \pi_t(s)$ is assumed to be measurable. For a fixed $s$, the map $t \mapsto g(t, s)$ is continuous because the curve $t \mapsto \pi_t(s)$ is absolutely continuous ($AC^2$) with respect to the Wasserstein metric. Since $g(t, s)$ is continuous in $t$ and measurable in $s$, it is a Carathéodory function and is therefore jointly measurable in $(t, s)$. Since $\phi$ and the time interval $I$ are bounded, Fubini-Tonelli's theorem ensures that integrating $g(t, s)$ with respect to the Lebesgue measure on $t$ yields a strictly Borel measurable function of $s$. Consequently, the mapping $s \mapsto \mu^s$ is measurable from $\Sstate$ into $\mathcal{M}_+(X)$ equipped with the weak-* topology.

\paragraph{4. Measurability of the Tangent Measure via the Borel Graph Theorem.}
We seek to define a globally measurable map $s \mapsto E^s$ taking values in the space of vector-valued Radon measures $\mathcal{M}(X; \mathbb{R}^d)$ on the space-time domain $X := I \times \Action$, such that $E^s = v_{\cdot}(s, \cdot) \mu^s$.

By the Cauchy-Schwarz inequality, the total variation of the target measure $E^s$ is bounded by the integrated metric derivative:
$$\|E^s\|_{TV} = \int_0^T \int_{\Action} \|v_t(s, a)\| \, \mathrm{d}\pi_t(s)(a) \, \mathrm{d}t \leq \int_0^T |\pi_t'|(s) \, \mathrm{d}t =: M(s).$$
Since $(t, s) \mapsto |\pi_t'|(s)$ is Borel measurable and $|\pi_t'|(s) \in L^2(I) \subset L^1(I)$, the bound $M: \Sstate \to [0, \infty)$ is Borel measurable.

For any integer $N \in \mathbb{N}$, we define the measurable sub-level set $\Sstate_N := \{ s \in \Sstate \mid M(s) \leq N \}$. On each slice $\Sstate_N$, the target measure $E^s$ is strictly confined to the closed ball $B_N := \{ E \in \mathcal{M}(X; \mathbb{R}^d) \mid \|E\|_{TV} \leq N \}$. Because the predual $C_0(X; \mathbb{R}^d)$ is separable, $B_N$, equipped with the weak-* topology, is a compact Polish space.

We define the single-valued map $f_N: \Sstate_N \to B_N$ by $f_N(s) = E^s$. To prove its measurability, we characterize its graph. We apply the abstract Benamou-Brenier functional defined by \cite[Proposition 5.18)]{santambrogio2015optimal} directly to our space-time domain $X$. The functional:
$$\mathscr{B}_2(\mu^s, E) := \sup \left\{ \int_X \alpha(t,a) \, \mathrm{d}\mu^s + \int_X \beta(t,a) \cdot \mathrm{d}E \right\},$$
where the supremum is over $(\alpha, \beta) \in C_0(X)$ satisfying the pointwise constraint $\alpha + \frac{1}{2}|\beta|^2 \leq 0$, is lower semi-continuous with respect to weak-* convergence. Crucially, if $E = v \mu^s$, then $\mathscr{B}_2(\mu^s, E) = \frac{1}{2} \int_X \|v\|^2 \mathrm{d}\mu^s$.

By the uniqueness of the minimal velocity field established in Step 1, $f_N(s)$ is the \textit{unique} measure in $B_N$ satisfying two exact constraints:
\begin{itemize}
    \item \textbf{Global Continuity Equation:} $L_k(s, E) = 0$ for all $k$, using a countable dense subset $\{\phi_k\} \subset C_c^\infty((0,T) \times \Action)$, where:
    $$L_k(s, E) := \int_X \partial_t \phi_k(t, a) \, \mathrm{d}\mu^s(t, a) + \int_X \nabla_a \phi_k(t, a) \cdot \mathrm{d}E(t, a).$$
    \item \textbf{Metric Derivative Bound:} $\mathscr{B}_2(\mu^s, E) \leq B(s)$, where $B(s) := \frac{1}{2} \int_I |\pi_t'|(s)^2 \, \mathrm{d}t$.
\end{itemize}

Because any valid solution $AC^2$ to the continuity equation inherently satisfies $\|v_t\|_{L^2(\pi_t)} \geq |\pi_t'|(s)$ for almost every $t$, enforcing the integrated upper bound $\mathscr{B}_2(\mu^s, E) \leq B(s)$ strictly forces the equality $\|v_t\|_{L^2(\pi_t)} = |\pi_t'|(s)$ almost everywhere. Thus, this single integrated bound uniquely identifies the minimal tangent field.

To characterize the graph of $f_N$ as a Borel set, we establish the joint measurability of the map $(s, E) \mapsto \mathscr{B}_2(\mu^s, E)$. For any fixed pair of test functions $(\alpha, \beta)$, the map $s \mapsto \int_X \alpha \, \mathrm{d}\mu^s$ is Borel measurable (established in Step 3), and $E \mapsto \int_X \beta \cdot \mathrm{d}E$ is weakly-* continuous. Their sum is a Carathéodory function, hence jointly measurable on $\Sstate_N \times B_N$. It remains to show that the supremum defining $\mathscr{B}_2$ can be restricted to a countable family. The admissible set $\mathscr{C} := \{(\alpha, \beta) \in C_0(X) \times C_0(X; \mathbb{R}^d) \mid \alpha + \frac{1}{2}|\beta|^2 \leq 0\}$ is not itself separable, but the supremum can nonetheless be computed over a countable subset. Indeed, since $X = I \times \Action$ is a locally compact separable metric space (being a product of a compact interval and a subset of $\mathbb{R}^{d_a}$), the Banach space $C_0(X)$ is separable, and so is $C_0(X) \times C_0(X; \mathbb{R}^d)$. Let $\{(\alpha_j, \beta_j)\}_{j \in \mathbb{N}}$ be a countable dense subset of the ambient product space. For each $j$, we define the \textit{tightened} pair $(\tilde{\alpha}_j, \beta_j)$ by setting $\tilde{\alpha}_j := \min(\alpha_j, -\frac{1}{2}|\beta_j|^2)$, which ensures $(\tilde{\alpha}_j, \beta_j) \in \mathscr{C}$. For any admissible $(\alpha, \beta) \in \mathscr{C}$ and any $\varepsilon > 0$, we can find $j$ such that $\|\alpha_j - \alpha\|_\infty + \|\beta_j - \beta\|_\infty < \varepsilon$. Since $\alpha \leq -\frac{1}{2}|\beta|^2$, we have $\tilde{\alpha}_j = \min(\alpha_j, -\frac{1}{2}|\beta_j|^2) \geq \alpha - C\varepsilon$ for a constant $C$ depending only on $\sup|\beta|$. Therefore, the integrals against $(\tilde{\alpha}_j, \beta_j)$ approximate those against $(\alpha, \beta)$ up to $O(\varepsilon(\|\mu^s\|_{TV} + \|E\|_{TV}))$, which is controlled on $B_N$. Taking $\varepsilon \to 0$ shows that $\mathscr{B}_2(\mu^s, E) = \sup_j \left\{ \int_X \tilde{\alpha}_j \, \mathrm{d}\mu^s + \int_X \beta_j \cdot \mathrm{d}E \right\}$. As a countable supremum of jointly measurable functions, $(s, E) \mapsto \mathscr{B}_2(\mu^s, E)$ is jointly Borel measurable.

Because both $\mathscr{B}_2$ and the bound $B(s)$ are jointly measurable, the graph of $f_N$:
$$\mathrm{Gr}(f_N) = \left( \bigcap_{k} L_k^{-1}(\{0\}) \right) \cap \left\{ (s, E) \in \Sstate_N \times B_N \mid \mathscr{B}_2(\mu^s, E) \leq B(s) \right\}$$
is a Borel subset of $\Sstate_N \times B_N$. Since $\Sstate_N$ is a Borel subset of the Polish space $\Sstate$, it is a standard Borel space. The space $B_N$, being a compact metrizable space with separable predual, is Polish. Since $f_N$ is a single-valued function from a standard Borel space into a Polish space with a Borel graph, the Borel Graph Theorem \cite[Theorem 12.28]{aliprantis2006infinite} ensures that $f_N$ is strictly Borel measurable. Extending this across $\Sstate = \bigcup_{N} \Sstate_N$ proves the map $s \mapsto E^s$ is globally Borel measurable.
\paragraph{5. Joint Measurability of the Density via Differentiation of Measures.}
We have established the existence of strictly Borel measurable maps $s \mapsto E^s$ and $s \mapsto \mu^s$ into the spaces of vector and scalar Radon measures respectively, with $E^s \ll \mu^s$ for every $s$. It remains to construct the Radon-Nikodym derivative $v(t, s, a) = \frac{\mathrm{d}E^s}{\mathrm{d}\mu^s}(t, a)$ such that it is jointly Borel measurable in $(t, s, a)$.

First, we must rigorously justify the evaluation of these measures on arbitrary Borel sets. For the vector measure $E^s$, we consider its scalar components $E^{s, (i)}$ for $i \in \{1, \dots, d\}$. By the definition of the weak-* topology, the evaluation map $s \mapsto \int_X f \, \mathrm{d}E^{s, (i)}$ is Borel measurable for any $f \in C_0(X)$. For any relatively compact open set $U \subset X$, its indicator function can be written as the pointwise limit of an increasing sequence of non-negative continuous functions $f_m \nearrow \mathbf{1}_U$. By the Dominated Convergence Theorem (dominated by the constant $1$, which is integrable with respect to the total variation measure $\|E^s\|_{TV}$), we have $E^{s, (i)}(U) = \lim_{m \to \infty} \int_X f_m \, \mathrm{d}E^{s, (i)}$. Since the pointwise limit of a sequence of measurable functions is measurable, the map $s \mapsto E^{s, (i)}(U)$ is Borel measurable.

By Dynkin's $\pi-\lambda$ theorem (the Monotone Class Theorem), since the collection of sets $B$ for which $s \mapsto E^{s, (i)}(B)$ is measurable forms a $\lambda$-system containing all relatively compact open sets (which form a $\pi$-system generating the Borel $\sigma$-algebra), the map $s \mapsto E^{s, (i)}(B)$ is Borel measurable for \textit{every} Borel set $B \in \mathcal{B}(X)$. The same property holds for the positive scalar measure $\mu^s$ via the Monotone Convergence Theorem.

We can now safely construct the density explicitly using the differentiation of measures. Let $X = I \times \Action$. Because $X$ is a Polish space, its Borel $\sigma$-algebra is countably generated. Therefore, there exists a sequence of refining countable partitions $\{\mathcal{P}_n\}_{n \in \mathbb{N}}$ of $X$ consisting of Borel sets such that $\sigma(\cup_n \mathcal{P}_n) = \mathcal{B}(X)$. 

For each $n$, enumerate the elements of the partition as $\mathcal{P}_n = \{P_{n,k}\}_{k=1}^\infty$. We define the sequence of approximate densities $v_n: \Sstate \times X \to \mathbb{R}^d$ explicitly as:
\[
    v_n(s, x) := \sum_{k=1}^\infty \mathbf{1}_{P_{n,k}}(x) \cdot F_{n,k}(s),
\]
where $F_{n,k}: \Sstate \to \mathbb{R}^d$ is defined by:
\[
    F_{n,k}(s) := \begin{cases} 
      \frac{E^s(P_{n,k})}{\mu^s(P_{n,k})} & \text{if } \mu^s(P_{n,k}) > 0, \\
      0 & \text{otherwise.}
   \end{cases}
\]

Because the evaluation maps $s \mapsto E^s(P_{n,k})$ and $s \mapsto \mu^s(P_{n,k})$ are Borel measurable, the vector-valued function $F_{n,k}(s)$ is Borel measurable in $s$. Viewed as functions on the product space $\Sstate \times X$, the map $(s, x) \mapsto F_{n,k}(s)$ is jointly Borel measurable, and the indicator function $(s, x) \mapsto \mathbf{1}_{P_{n,k}}(x)$ is jointly Borel measurable. Because $v_n$ is a countable sum of jointly measurable functions, $v_n$ itself is jointly Borel measurable on $\Sstate \times X$.

By the Martingale Convergence Theorem \cite[Theorem 13.33]{koralov2007theory} applied to the differentiation of measures along refining partitions, the sequence of approximate densities $v_n$ converges $\mu^s$-almost everywhere to the true Radon-Nikodym derivative. To ensure the final function is rigorously measurable everywhere without invoking undefined limits on vectors, we define the true density component by component using the pointwise limit superior:
\[
    v^{(i)}(t, s, a) := \limsup_{n \to \infty} v_n^{(i)}(s, (t, a)) \quad \text{for each coordinate } i \in \{1, \dots, d\}.
\]
Since the limit superior of a sequence of jointly measurable scalar functions is jointly measurable, the vector-valued map $(t, s, a) \mapsto v(t, s, a)$ is strictly Borel measurable with respect to the product $\sigma$-algebra $\mathcal{B}(I) \otimes \mathcal{B}(\Sstate) \otimes \mathcal{B}(\Action)$. For every $s \in \Sstate$, this function satisfies $E^s = v(\cdot, s, \cdot)\mu^s$, perfectly matching the minimal vector field and concluding the proof.

\end{proof}

\subsection{Riemannian Structure (Local Metric)}

Having identified the tangent space, we now define the global Riemannian metric on $\Pi$. Unlike a flat metric, our inner product depends on the frequency of state visitation, specifically the stationary distribution $\mu_\pi$ induced by the current policy for which \textbf{we assume existence and uniqueness for all policies} (The existence and uniqueness are guaranteed under some conditions by Theorems from Section \ref{sec:induced_dyn_erg}).

For two tangent vectors $\xi, \eta \in T_{\pi}\Pi$ (identified with their respective velocity fields $v^\xi$ and $v^\eta$), the local inner product is defined as the average of the local inner products, weighted by the state occupancy $\mu_\pi$:
\begin{equation}\label{eq:def_metric}
    \langle \xi, \eta \rangle_{\pi} := \int_{\mathcal{S}} \langle \xi(s), \eta(s) \rangle_{T_{\pi(s)}\mathcal{P}_{2}(\mathcal{A})} \, d\mu_{\pi}(s).
\end{equation}

The internal term $\langle \cdot, \cdot \rangle_{T_{\pi(s)}\mathcal{P}_{2}(\mathcal{A})}$ corresponds to the standard inner product on the Wasserstein manifold. Using the identification via the continuity equation established above, this is concretely expressed as:
\begin{equation}
    \langle \xi(s), \eta(s) \rangle_{T_{\pi(s)}\mathcal{P}_{2}(\mathcal{A})} = \int_{\mathcal{A}} \langle v^\xi(s, a), v^\eta(s, a) \rangle_{\mathbb{R}^{d_a}} \, d(\pi(s))(a).
\end{equation}

This metric thus combines the geometry of the action space (via local optimal transport) and the dynamics of the environment (via integration over the invariant measure $\mu_\pi$).

Strictly speaking, because the integration is with respect to $\mu_\pi$, this bilinear form initially defines a seminorm. To rigorously obtain a non-degenerate Riemannian metric, we consider the joint probability measure on states and actions $\mathrm{d}\gamma_\pi(s, a) := \mathrm{d}\pi(s)(a)\mathrm{d}\mu_\pi(s)$. The global tangent space $T_\pi \Pi$ is formally defined as the closure of the space of action-gradients of smooth potentials in $L^2(\gamma_\pi)$:
\begin{equation}
    T_\pi \Pi := \overline{ \left\{ \nabla_a \phi \mid \phi \in C_c^\infty(\Sstate \times \Action) \right\} }^{L^2(\Sstate \times \Action; \mathbb{R}^{d_a}, \gamma_\pi)}.
\end{equation}
This completion inherently identifies vector fields that are equal $\gamma_\pi$-almost everywhere, ensuring the metric is positive definite.

\subsection{Gradient of the Objective Function}

We now turn to the optimization of a standard Reinforcement Learning objective. Let $c: \Sstate \times \Action \to \mathbb{R}$ be a differentiable cost function. The goal is to minimize the expected long-term average cost:
\[
    J(\pi) := \int_{\Sstate} \left( \int_{\Action} c(s, a) \, \mathrm{d}(\pi(s))(a) \right) \, \mathrm{d}\mu_\pi(s).
\]

To compute the gradient of $J$ with respect to the Wasserstein metric structure defined previously, we must account for the coupling between the policy update and the stationary distribution. We employ \textbf{Otto calculus} to differentiate functionals over probability spaces by formally relating variations in measure to vector fields. Hence, from now on and until the end of this work, all functions are assumed to be smooth in order to perform derivations or integration by parts.

\subsubsection{Variation of the Objective}

Consider a perturbation of the policy $\pi$ along a tangent vector $\xi \in T_\pi \Pi$. The variation of the objective $J$ involves two terms: the direct change in the cost expectation and the indirect change arising from the shift in the stationary distribution $\mu_\pi$:
\begin{equation}
    \delta J(\xi) = \underbrace{\int_{\Sstate} \int_{\Action} c(s, a) \, \delta \pi(s)(\mathrm{d}a) \, \mathrm{d}\mu_\pi(s)}_{\text{Direct effect}} + \underbrace{\int_{\Sstate} C_\pi(s) \, \delta \mu_\pi(\mathrm{d}s)}_{\text{Geometric effect}},
    \label{eq:total_var}
\end{equation}
where $C_\pi(s) := \int_{\Action} c(s, a) \, \mathrm{d}(\pi(s))(a)$ is the expected cost at state $s$. The term involving $\delta \mu_\pi$ is implicit; to resolve it, we utilize the adjoint method involving the transition kernel $\mathcal{T}$.

\subsubsection{Sensitivity Analysis via the Poisson Equation}

We first establish the relationship between the variation of the stationary measure and the transition dynamics using the Poisson equation.

\begin{proposition}[Adjoint Sensitivity]\label{prop:ajoint_sensitivity}
Let $V^\pi: \Sstate \to \mathbb{R}$ be a solution to the Poisson equation:
\begin{equation}
    (I - P^\pi)V^\pi = C_\pi - J(\pi)\mathbf{1}.
    \label{eq:poisson}
\end{equation}
Then, the variation of the stationary distribution satisfies:
\[
    \langle \delta \mu_\pi, C_\pi \rangle_{\Sstate} = \langle \mu_\pi (\delta P^\pi), V^\pi \rangle_{\Sstate}.
\]
\end{proposition}

\begin{proof}
    The stationary distribution $\mu_\pi$ is the fixed point of the pushforward operator: $\mu_\pi P^\pi = \mu_\pi$. Taking the first-order variation yields:
    \[
        (\delta \mu_\pi) P^\pi + \mu_\pi (\delta P^\pi) = \delta \mu_\pi.
    \]
    We rearrange terms to group the variation of the measure $\delta \mu_\pi$:
    \[
        \delta \mu_\pi (I - P^\pi) = \mu_\pi (\delta P^\pi).
    \]
    This is an equality of signed measures. We apply both sides to the function $V^\pi$ using the duality pairing $\langle \cdot, \cdot \rangle_{\Sstate}$:
    \begin{equation}
        \langle \delta \mu_\pi (I - P^\pi), V^\pi \rangle_{\Sstate} = \langle \mu_\pi (\delta P^\pi), V^\pi \rangle_{\Sstate}.
        \label{eq:adjoint_step}
    \end{equation}
    
We use the adjoint property of the operator $(I - P^\pi)$:
    \[
        \langle \delta \mu_\pi, (I - P^\pi)V^\pi \rangle_{\Sstate} = \langle \delta \mu_\pi (I - P^\pi), V^\pi \rangle_{\Sstate}.
    \]
    Substituting Eq.~\eqref{eq:adjoint_step} into this result gives:
    \[
        \langle \delta \mu_\pi, C_\pi \rangle_{\Sstate} = \langle \mu_\pi (\delta P^\pi), V^\pi \rangle_{\Sstate}
    \]
    where we used the fact that $\langle \delta \mu_\pi, J(\pi)\mathbf{1} \rangle_{\Sstate} = 0 $ as $\mu_\pi$ is a probability measure and so $\delta \mu_\pi$ has zero mean.  
\end{proof}

\begin{remark}[Well-posedness of the Poisson Equation]
    Strictly speaking, the operator $L = I - P^\pi$ is not invertible on the full function space (e.g. the Banach space $L^\infty(\Sstate)$) because $P^\pi \mathbf{1} = \mathbf{1}$, implying that constant functions lie in its kernel. To rigorously ensure the solvability of the Poisson equation uniformly along a policy path, we make the strong but standard assumption that for any policy $\pi$, the transition operator $P^\pi$ is compact. This allows us to apply the Fredholm alternative. Furthermore, we assume that the kernel $\ker(I - P^\pi)$ is strictly reduced to constant functions for all $\pi$.

    By the Fredholm alternative, the linear equation $(I - P^\pi)V = g$ admits a bounded solution if and only if the right-hand side $g$ is orthogonal to the kernel of the adjoint operator $L^*$. The adjoint operator acts on signed measures, and its kernel consists of $\nu$ such that $\nu(I - P^\pi) = 0$, or equivalently $\nu = \nu P^\pi$. Under our assumptions, this kernel is one-dimensional and spanned by the unique invariant measure $\mu_\pi$. Consequently, the solvability condition requires:
    \[
        \langle \mu_\pi, g \rangle_{\Sstate} = \int_{\Sstate} g(s) \, \mathrm{d}\mu_\pi(s) = 0.
    \]
    Since the raw cost $C_\pi$ generally does not have zero mean, the standard Poisson equation in reinforcement learning centers the target by subtracting the average reward $J(\pi) = \langle \mu_\pi, C_\pi \rangle$. The equation becomes:
    \[
        (I - P^\pi)V^\pi = C_\pi - J(\pi)\mathbf{1}.
    \]
   Under these conditions, the operator is boundedly invertible on the subspace of mean-zero functions, determining $V^\pi$ uniquely up to an additive constant (usually fixed by enforcing $\int V^\pi \, \mathrm{d}\mu_\pi = 0$). 
\end{remark}

\subsubsection{The Q-Function Representation}

We can now express the total variation $\delta J$ solely in terms of expectations over the current stationary distribution $\mu_\pi$, eliminating the implicit term $\delta \mu_\pi$.

\begin{proposition}[Total Variation in terms of Q-function]\label{prop:TV_J}
The total variation of the objective is given by:
\[
    \delta J = \int_{\Sstate} \left( \int_{\Action} Q^\pi(s, a) \, \delta \pi(s)(\mathrm{d}a) \right) \, \mathrm{d}\mu_\pi(s),
\]
where the state-action value function (Q-function) is defined as:
\[
    Q^\pi(s, a) := c(s, a) + \int_{\Sstate} V^\pi(s') \mathcal{T}(s, a, \mathrm{d}s').
\]
\end{proposition}

\begin{proof}
    Substituting the result from Proposition \ref{prop:ajoint_sensitivity} into equation \eqref{eq:total_var}, we have:
    \[
        \delta J = \langle \mu_\pi, \delta C_\pi \rangle + \langle \mu_\pi, (\delta P^\pi) V^\pi \rangle = \langle \mu_\pi, \delta C_\pi + (\delta P^\pi) V^\pi \rangle.
    \]
    We must explicitly define the variations $\delta C_\pi$ and $\delta P^\pi$ in terms of the policy variation $\delta \pi$.
    
    1. \textbf{Variation of Cost:} Recall $C_\pi(s) = \int_\Action c(s, a) \pi(s)(\mathrm{d}a)$. Its variation is:
    \[
        \delta C_\pi(s) = \int_\Action c(s, a) \, \delta \pi(s)(\mathrm{d}a).
    \]
    
    2. \textbf{Variation of Kernel:} Recall $P^\pi(s, \mathrm{d}s') = \int_\Action \mathcal{T}(s, a, \mathrm{d}s') \pi(s)(\mathrm{d}a)$.
    The operator $(\delta P^\pi) V^\pi$ at state $s$ is:
    \[
        ((\delta P^\pi) V^\pi)(s) = \int_{\Sstate} V^\pi(s') \, \delta P^\pi(s, \mathrm{d}s') = \int_{\Sstate} V^\pi(s') \int_{\Action} \mathcal{T}(s, a, \mathrm{d}s') \, \delta \pi(s)(\mathrm{d}a).
    \]
    By Fubini's theorem, we swap the integrals:
    \[
        ((\delta P^\pi) V^\pi)(s) = \int_{\Action} \left( \int_{\Sstate} V^\pi(s') \mathcal{T}(s, a, \mathrm{d}s') \right) \, \delta \pi(s)(\mathrm{d}a).
    \]
    
    Summing these two terms:
    \begin{align*}
        \delta C_\pi(s) + ((\delta P^\pi) V^\pi)(s) &= \int_{\Action} c(s, a) \, \delta \pi(s)(\mathrm{d}a) + \int_{\Action} \left( \int_{\Sstate} V^\pi(s') \mathcal{T}(s, a, \mathrm{d}s') \right) \, \delta \pi(s)(\mathrm{d}a) \\
        &= \int_{\Action} \left[ c(s, a) + \int_{\Sstate} V^\pi(s') \mathcal{T}(s, a, \mathrm{d}s') \right] \, \delta \pi(s)(\mathrm{d}a).
    \end{align*}
    The term in the brackets is exactly the definition of $Q^\pi(s, a)$. Thus:
    \[
        \delta J = \int_{\Sstate} \left( \int_{\Action} Q^\pi(s, a) \, \delta \pi(s)(\mathrm{d}a) \right) \, \mathrm{d}\mu_\pi(s).
    \]
\end{proof}

The defined $Q^\pi$ satisfies the Bellman equation as stated in Proposition \ref{prop:Bellman}.
\begin{proposition}[Bellman Equation for $Q^\pi$]\label{prop:Bellman}
The state-action value function $Q^\pi$ satisfies the following Bellman equation:
\[
    Q^\pi(s, a) = c(s, a) - J(\pi) + \int_{\Sstate} \int_{\Action} Q^\pi(s', a') \, \mathrm{d}(\pi(s'))(a') \, \mathcal{T}(s, a, \mathrm{d}s').
\]
Furthermore, it is related to the value function $V^\pi$ via the consistency relation:
\[
    V^\pi(s) = \int_{\Action} Q^\pi(s, a) \, \mathrm{d}(\pi(s))(a) - J(\pi).
\]
\end{proposition}

\begin{proof}
    We start from the Poisson equation satisfied by $V^\pi$ (Eq. \ref{eq:poisson}):
    \[
        V^\pi(s) - (P^\pi V^\pi)(s) = C_\pi(s) - J(\pi).
    \]
    Rearranging terms to isolate $V^\pi(s)$:
    \[
        V^\pi(s) = C_\pi(s) - J(\pi) + (P^\pi V^\pi)(s).
    \]
    Recalling the definitions of $C_\pi$ and the operator $P^\pi$ acting on $V^\pi$:
    \[
        C_\pi(s) = \int_{\Action} c(s, a) \, \mathrm{d}(\pi(s))(a), \quad (P^\pi V^\pi)(s) = \int_{\Action} \int_{\Sstate} V^\pi(s') \mathcal{T}(s, a, \mathrm{d}s') \, \mathrm{d}(\pi(s))(a).
    \]
    Substituting these expressions, we obtain:
    \[
        V^\pi(s) = \int_{\Action} \left( c(s, a) + \int_{\Sstate} V^\pi(s') \mathcal{T}(s, a, \mathrm{d}s') \right) \mathrm{d}(\pi(s))(a) - J(\pi).
    \]
    The term in parentheses corresponds exactly to the definition of $Q^\pi(s, a)$ given in Proposition \ref{prop:TV_J}. Thus, we have established the consistency relation:
    \begin{equation}
        V^\pi(s) = \int_{\Action} Q^\pi(s, a) \, \mathrm{d}(\pi(s))(a) - J(\pi).
        \label{eq:v_q_relation}
    \end{equation}
    
    To obtain the Bellman equation for $Q^\pi$, we plug this expression for $V^\pi(s')$ into the initial definition of $Q^\pi(s, a)$:
    \begin{align*}
        Q^\pi(s, a) &= c(s, a) + \int_{\Sstate} V^\pi(s') \, \mathcal{T}(s, a, \mathrm{d}s') \\
        &= c(s, a) + \int_{\Sstate} \left( \int_{\Action} Q^\pi(s', a') \, \mathrm{d}(\pi(s'))(a') - J(\pi) \right) \mathcal{T}(s, a, \mathrm{d}s') \\
        &= c(s, a) - J(\pi) + \int_{\Sstate} \left( \int_{\Action} Q^\pi(s', a') \, \mathrm{d}(\pi(s'))(a') \right) \mathcal{T}(s, a, \mathrm{d}s'),
    \end{align*}
    where the constant $J(\pi)$ is pulled out using the fact that $\int_{\Sstate} \mathcal{T}(s, a, \mathrm{d}s') = 1$.
\end{proof}

\subsubsection{Derivation of the Riemannian Gradient}

Finally, we link this variation to the geometry of the Wasserstein space to derive the gradient using Otto calculus.

\begin{proposition}[Wasserstein Policy Gradient]\label{prop:wasserstein_grad}
    If variations $\delta \pi(s)$ are generated by a vector field $v(s, \cdot)$ via the continuity equation $\partial_t \pi + \nabla \cdot (\pi v) = 0$ supplemented with zero flux boundary condition, then the gradient of the objective with respect to the stationary-weighted Wasserstein metric is:
    \[
        \mathrm{grad}_\pi J(s, a) = \nabla_a Q^\pi(s, a).
    \]
\end{proposition}

\begin{proof}
    Consider the inner integral $I(s) = \int_{\Action} Q^\pi(s, a) \, \delta \pi(s)(\mathrm{d}a)$ from Proposition \ref{prop:TV_J}.
    In the context of Otto calculus, a variation $\delta \pi$ in the tangent space of $\Ptwo(\Action)$ is identified with a vector field $v$ such that $\delta \pi = -\nabla_a \cdot (\pi v)$ (in the distributional sense).
    Substituting this into the integral:
    \[
        I(s) = - \int_{\Action} Q^\pi(s, a) \nabla_a \cdot (\pi(s)(a) v(s, a)) \, \mathrm{d}a.
    \]
    Applying integration by parts (Green's formula) on the domain $\Action$:
    \[
        I(s) = \int_{\Action} \langle \nabla_a Q^\pi(s, a), v(s, a) \rangle_{\mathbb{R}^{d_a}} \, \pi(s)(\mathrm{d}a).
    \]
    This is precisely the $L^2(\pi(s))$ inner product between the gradient of $Q$ and the velocity field $v$:
    \[
        I(s) = \langle \nabla_a Q^\pi(s, \cdot), v(s, \cdot) \rangle_{T_{\pi(s)}\Ptwo}.
    \]
    Integrating over $\Sstate$ with respect to $\mu_\pi$, we recover the global Riemannian metric defined in Equation \eqref{eq:def_metric}:
    \[
        \delta J = \int_{\Sstate} \langle \nabla_a Q^\pi(s, \cdot), v(s, \cdot) \rangle_{T_{\pi(s)}\Ptwo} \, \mathrm{d}\mu_\pi(s) = \langle \nabla_a Q^\pi, v \rangle_{\pi}.
    \]
    By the definition of the Riemannian gradient $\delta J(v) = \langle \text{grad}_\pi J, v \rangle_\pi$, we identify:
    \[
        \text{grad}_\pi J(s, a) = \nabla_a Q^\pi(s, a).
    \]
\end{proof}

\begin{remark}
We note that a similar action-space velocity $\nabla_a Q^\pi$ was recently derived by \cite{pfau2025wasserstein} for Wasserstein Policy Optimization (WPO) in the discounted setting. Our derivation here complements their work by formally establishing this geometry in the average-cost continuous-time setting using Otto calculus.
\end{remark}

\section{Hessian Analysis and Geodesic Convexity}\label{sec:Hessian}

To investigate the convexity properties of the objective function $J(\pi)$, we perform a second-order analysis using Otto calculus. 

\subsection{Derivation of the Hessian}

To rigorously analyze the geometry of the policy space, we must identify its geodesics. Unlike the standard Wasserstein space, the metric on $\Pi$ is weighted by the policy-dependent stationary distribution $\mu_{\pi_t}$. Consequently, the geodesics do not simply follow state-wise Hamilton-Jacobi equations but incorporate non-local curvature terms arising from the variation of the stationary measure.

\begin{theorem}[Geodesic Equations in Policy Space]
\label{th:geodesic_equations}
Let $(\pi_t)_{t \in I}$ be a constant-speed geodesic curve in $\Pi$ parameterized by $t \in [0, 1]$. Assume that for all $t \in I$, the induced stationary distribution admits a strictly positive density, i.e., $\mu_{\pi_t}(s) > 0$ for all $s \in \Sstate$. Let $v_t(s, a) = \nabla_a \phi_t(s, a)$ be the associated tangent velocity field generated by the potential $\phi_t$. The geodesic evolution is governed by the following coupled system:
\begin{align}
    \partial_t \pi_t + \nabla_a \cdot (\pi_t \nabla_a \phi_t) &= 0 \quad \text{(Continuity Equation)} \label{eq:geo_continuity} \\
    \partial_t \phi_t + \frac{1}{2} \| \nabla_a \phi_t \|^2 + \frac{\dot{\mu}_{\pi_t}}{\mu_{\pi_t}} \phi_t - U_t &= 0 \quad \text{(Modified Hamilton-Jacobi Equation)} \label{eq:geo_hj}
\end{align}
where $\dot{\mu}_{\pi_t} = \partial_t \mu_{\pi_t}$ is the Eulerian time derivative of the stationary distribution along the policy flow, and $U_t(s,a)$ is the dynamic anticipation potential defined as:
\begin{equation}
    U_t(s, a) = \int_{\Sstate} W_t(s') \mathcal{T}(s, a, \mathrm{d}s').
\end{equation}
Here, $W_t$ is the unique (mean-zero) solution to the following Poisson equation:
\begin{equation}
    (I - P^{\pi_t})W_t = \frac{1}{2} \mathcal{E}_t - K(t)\mathbf{1},
\end{equation}
with the local kinetic energy $\mathcal{E}_t(s) := \int_{\Action} \| \nabla_a \phi_t(s, a) \|^2 \mathrm{d}\pi_t(s)(a)$ and the global kinetic energy $K(t) := \int_{\Sstate} \frac{1}{2} \mathcal{E}_t(s) \mathrm{d}\mu_{\pi_t}(s)$.
\end{theorem}

\begin{proof}
By the principle of least action, a geodesic minimizes the global kinetic energy action $\mathcal{J} = \int_0^1 K(t) \, \mathrm{d}t$. We enforce the continuity equation constraint using a state-action Lagrange multiplier $\psi_t(s, a)$. The unconstrained augmented action is:
\[
    \mathcal{L}(\pi, v, \psi) = \int_0^1 \int_{\Sstate} \left( \mu_{\pi_t}(s) \frac{1}{2} \mathcal{E}_t(s) + \int_{\Action} \psi_t(s, a) \big(\partial_t \pi_t + \nabla_a \cdot (\pi_t v_t)\big) \mathrm{d}a \right) \mathrm{d}s \, \mathrm{d}t.
\]
For mathematical convenience, since $\mu_{\pi_t}(s) > 0$ almost everywhere on its support, we apply the change of variable $\psi_t(s, a) = \mu_{\pi_t}(s) \phi_t(s, a)$.

\textbf{1. First variation with respect to velocity $v_t$:}
Perturbing the velocity field by $\delta v_t$ yields:
\[
    \delta_v \mathcal{L} = \int_0^1 \int_{\Sstate} \int_{\Action} \left[ \mu_{\pi_t} \langle v_t, \delta v_t \rangle \pi_t - \langle \nabla_a \psi_t, \delta v_t \rangle \pi_t \right] \mathrm{d}a \, \mathrm{d}s \, \mathrm{d}t = 0.
\]
This implies $\mu_{\pi_t} v_t = \nabla_a \psi_t$. Substituting $\psi_t = \mu_{\pi_t} \phi_t$ and noting that $\nabla_a \mu_{\pi_t} = 0$, we recover the standard optimal transport relation $v_t = \nabla_a \phi_t$.

\textbf{2. First variation with respect to the policy $\pi_t$:}
We apply a perturbation $\delta \pi_t$. Crucially, this induces a variation of the stationary measure $\delta \mu_{\pi_t}$. The variation of the Lagrangian decomposes into the constraint variation and the kinetic energy variation.

\textit{a. Constraint Term Variation:}
Because $\psi_t$ is the independent multiplier, the constraint variation solely acts on the derivatives of the policy. Integrating by parts over time and action:
\begin{align*}
    \delta_\pi(\text{Constraint}) &= \int_0^1 \int_{\Sstate} \int_{\Action} \psi_t \big( \partial_t \delta\pi_t + \nabla_a \cdot (\delta\pi_t v_t) \big) \, \mathrm{d}a \, \mathrm{d}s \, \mathrm{d}t \\
    &= - \int_0^1 \int_{\Sstate} \int_{\Action} \big( \partial_t \psi_t + \langle \nabla_a \psi_t, v_t \rangle \big) \delta\pi_t \, \mathrm{d}a \, \mathrm{d}s \, \mathrm{d}t.
\end{align*}

\textit{b. Kinetic Energy Term Variation ($\delta_\pi \mathcal{J}$):}
Using the product rule, the variation of the global kinetic energy explicitly incorporates the variation of the stationary distribution $\delta \mu_{\pi_t}$:
\[
    \delta_\pi \mathcal{J} = \underbrace{\int_0^1 \int_{\Sstate} \frac{1}{2} \mathcal{E}_t(s) \delta\mu_{\pi_t}(\mathrm{d}s) \, \mathrm{d}t}_{\text{Variation of the measure}} + \underbrace{\int_0^1 \int_{\Sstate} \mu_{\pi_t}(s) \frac{1}{2} \delta \mathcal{E}_t(s) \, \mathrm{d}s \, \mathrm{d}t}_{\text{Local variation}}.
\]
By Proposition \ref{prop:ajoint_sensitivity}, the variation of the measure is mapped to the variation of the transition kernel via the kinetic value function $W_t$:
\[
    \langle \delta \mu_{\pi_t}, \frac{1}{2} \mathcal{E}_t \rangle = \langle \mu_{\pi_t} \delta P^{\pi_t}, W_t \rangle = \int_{\Sstate} \mu_{\pi_t}(s) \int_{\Action} \left[ \int_{\Sstate} W_t(s') \mathcal{T}(s, a, \mathrm{d}s') \right] \delta\pi_t(s)(\mathrm{d}a) \, \mathrm{d}s.
\]
The bracketed term is exactly $U_t(s,a)$. The local variation of $\mathcal{E}_t(s)$ simply yields $\frac{1}{2}\|v_t\|^2 \delta\pi_t$. Summing these components:
\[
    \delta_\pi \mathcal{J} = \int_0^1 \int_{\Sstate} \int_{\Action} \mu_{\pi_t}(s) \left[ U_t(s,a) + \frac{1}{2}\|v_t(s,a)\|^2 \right] \delta\pi_t(s)(\mathrm{d}a) \, \mathrm{d}s \, \mathrm{d}t.
\]

\textbf{3. Synthesis and Equation Identification:}
Setting the total variation $\delta_\pi \mathcal{L} = 0$ yields the pointwise equality:
\[
    - \partial_t \psi_t - \langle \nabla_a \psi_t, v_t \rangle + \mu_{\pi_t} \left( \frac{1}{2}\|v_t\|^2 + U_t \right) = 0.
\]
Finally, we substitute back $\psi_t = \mu_{\pi_t} \phi_t$. Using the product rule, $\partial_t (\mu_{\pi_t} \phi_t) = \dot{\mu}_{\pi_t} \phi_t + \mu_{\pi_t} \partial_t \phi_t$, and $\nabla_a (\mu_{\pi_t} \phi_t) = \mu_{\pi_t} \nabla_a \phi_t$:
\[
    - \dot{\mu}_{\pi_t} \phi_t - \mu_{\pi_t} \partial_t \phi_t - \mu_{\pi_t} \langle \nabla_a \phi_t, v_t \rangle + \mu_{\pi_t} \left( \frac{1}{2}\|v_t\|^2 + U_t \right) = 0.
\]
Since $v_t = \nabla_a \phi_t$, the term $\langle \nabla_a \phi_t, v_t \rangle$ evaluates exactly to $\|v_t\|^2$. Dividing the entire equation by $\mu_{\pi_t} > 0$ and isolating $\partial_t \phi_t$, we obtain the modified Hamilton-Jacobi equation \eqref{eq:geo_hj}.
\end{proof}

The Riemannian Hessian $\text{Hess}_\pi J(\xi, \xi)$ is defined as the second time derivative of the objective evaluated at $t=0$:
\[
    \text{Hess}_\pi J(\xi, \xi) := \left. \frac{d^2}{dt^2} J(\pi_t) \right|_{t=0}.
\]

Recall from Proposition \ref{prop:wasserstein_grad} that the first derivative is given by the pairing of the gradient and the velocity:
\[
    \frac{d}{dt} J(\pi_t) = \langle \text{grad}_{\pi_t} J, \xi_t \rangle_{\pi_t} = \int_{\Sstate} \int_{\Action} \langle \nabla_a Q^{\pi_t}(s, a), v_t(s, a) \rangle \, \mathrm{d}\pi_t(s)(a) \, \mathrm{d}\mu_{\pi_t}(s).
\]

Differentiating this expression with respect to $t$ requires applying the product rule to the integral structure. This yields two distinct terms corresponding to the variation of the stationary measure and the variation of the inner expectation:
\begin{equation}
    \frac{d^2 J}{dt^2} = \underbrace{\int_{\Sstate} \left( \int_{\Action} \langle \nabla_a Q^{\pi}, v_t \rangle \, \mathrm{d}\pi(s) \right) \, \frac{d}{dt}\mu_{\pi_t}(\mathrm{d}s)}_{T_1 (\text{Measure Variation})} 
    + \underbrace{\int_{\Sstate} \frac{d}{dt} \left( \int_{\Action} \langle \nabla_a Q^{\pi_t}, v_t \rangle \, \mathrm{d}\pi_t(s) \right) \, \mathrm{d}\mu_\pi(s)}_{T_2 (\text{Inner Variation})}.
\end{equation}

\paragraph{Notation for Evaluation at $t=0$.} Throughout the remainder of this section, since we evaluate these derivatives at the origin of the geodesic where $\pi_0 = \pi$, we adopt the subscript $0$ to denote quantities evaluated at $t=0$. Specifically:
\begin{itemize}
    \item $v_0 := v_t|_{t=0}$ is the velocity field generating the tangent vector $\xi$.
    \item $U_0 := U_t|_{t=0}$ is the dynamic anticipation potential at the base policy.
    \item $\mu_{\pi_0}$ trivially reduces to the stationary distribution $\mu_\pi$.
\end{itemize}

\subsection{Analysis of the Inner Variation ($T_2$)}

Here we express the term $T_2$ in a more precise way.

\begin{proposition}[Computation of Inner Variation $T_2$]\label{prop:T_2term}
Let $T_2$ be the component of the Hessian arising from the variation of the inner product along the geodesic in the action space:
\[
    T_2 := \int_{\Sstate}  \left. \frac{d}{dt} \right|_{t=0} \left( \int_{\Action} \langle \nabla_a Q^{\pi_t}(s, a), v_t(s, a) \rangle \, \mathrm{d}\pi_t(s)(a) \right) \mathrm{d}\mu_\pi(s).
\]
Assuming that $(\pi_t)_{t}$ is a constant-speed geodesic in the policy space $\Pi$ generated by the vector field $v_t$ (as defined in Theorem \ref{th:geodesic_equations}), $T_2$ is given by:
\[
    T_2 = \int_{\Sstate} \int_{\Action} \left[ \text{Hess}_a Q^\pi(s, a)(v_0, v_0) + \langle \nabla_a \dot{Q}^\pi, v_0 \rangle + \langle \nabla_a Q^\pi, \nabla_a U_0 \rangle \right] \mathrm{d}\pi(s)(a) \, \mathrm{d}\mu_\pi(s) - T_1,
\]
where $\dot{Q}^\pi = \left. \partial_t Q^{\pi_t} \right|_{t=0}$.
\end{proposition}

\begin{proof}
    We focus on the inner integral for a fixed state $s$. Let $I(t)$ denote this quantity. To compute the time derivative $\dot{I}$, we employ the Reynolds Transport Theorem for measures evolving according to the continuity equation $\partial_t \pi_t + \nabla_a \cdot (\pi_t v_t) = 0$. For the time-dependent observable $f_t(a) := \langle \nabla_a Q^{\pi_t}(s, a), v_t(s, a) \rangle$, the derivative is:
    \begin{equation}
        \frac{d}{dt} \mathbb{E}_{\pi_t}[f_t] = \mathbb{E}_{\pi_t} \left[ \frac{D f_t}{Dt} \right] = \mathbb{E}_{\pi_t} \left[ \partial_t f_t + \langle \nabla_a f_t, v_t \rangle \right].
        \label{eq:transport_theorem}
    \end{equation}

    \textbf{1. The Partial and Convective Derivatives:}
    Applying the product rule to the inner product, the partial time derivative is:
    \begin{equation}
        \partial_t f_t = \langle \nabla_a \dot{Q}^{\pi_t}, v_t \rangle + \langle \nabla_a Q^{\pi_t}, \partial_t v_t \rangle.
        \label{eq:partial_t}
    \end{equation}
    Using the vector calculus identity $\nabla \langle A, B \rangle = (A \cdot \nabla)B + (B \cdot \nabla)A$ (ignoring curl terms as $v_t$ is a gradient field), we expand the gradient of $f_t$:
    \[
        \nabla_a f_t = (\nabla_a Q^{\pi_t} \cdot \nabla_a) v_t + (v_t \cdot \nabla_a) \nabla_a Q^{\pi_t}.
    \]
    The second term yields $\text{Hess}_a Q^{\pi_t} \cdot v_t$. For the first term, because $v_t = \nabla_a \phi_t$, its Jacobian is symmetric, allowing the rearrangement $\langle (\nabla_a Q^{\pi_t} \cdot \nabla_a) v_t, v_t \rangle = \langle \nabla_a Q^{\pi_t}, (v_t \cdot \nabla_a) v_t \rangle$. Thus:
    \begin{equation}
        \langle \nabla_a f_t, v_t \rangle = \langle \nabla_a Q^{\pi_t}, (v_t \cdot \nabla_a) v_t \rangle + \text{Hess}_a Q^{\pi_t}(v_t, v_t).
        \label{eq:convective}
    \end{equation}

    \textbf{2. Applying the Policy Space Geodesic Condition:}
    Summing Equations \eqref{eq:partial_t} and \eqref{eq:convective}, we group the acceleration terms of the vector field:
    \[
        \frac{D f_t}{Dt} = \langle \nabla_a \dot{Q}^{\pi_t}, v_t \rangle + \text{Hess}_a Q^{\pi_t}(v_t, v_t) + \langle \nabla_a Q^{\pi_t}, \underbrace{\partial_t v_t + (v_t \cdot \nabla_a) v_t}_{\text{Acceleration}} \rangle.
    \]
    Unlike a standard Wasserstein geodesic where the acceleration vanishes, $(\pi_t)$ is a geodesic in the policy space $\Pi$. By Theorem \ref{th:geodesic_equations}, its potential satisfies the modified Hamilton-Jacobi equation:
    \[
        \partial_t \phi_t + \frac{1}{2} \|v_t\|^2 + \frac{\dot{\mu}_{\pi_t}}{\mu_{\pi_t}} \phi_t - U_t = 0.
    \]
    Taking the gradient with respect to $a$ (and noting $\nabla_a(\frac{1}{2}\|v_t\|^2) = (v_t \cdot \nabla_a)v_t$), we find the true acceleration:
    \[
        \partial_t v_t + (v_t \cdot \nabla_a) v_t = \nabla_a U_t - \frac{\dot{\mu}_{\pi_t}}{\mu_{\pi_t}} v_t.
    \]
    Substituting this back yields:
    \[
        \frac{D f_t}{Dt} = \text{Hess}_a Q^{\pi_t}(v_t, v_t) + \langle \nabla_a \dot{Q}^{\pi_t}, v_t \rangle + \langle \nabla_a Q^{\pi_t}, \nabla_a U_t \rangle - \frac{\dot{\mu}_{\pi_t}}{\mu_{\pi_t}} \langle \nabla_a Q^{\pi_t}, v_t \rangle.
    \]

    \textbf{3. Integration and Cancellation:}
    We integrate this material derivative against $\pi_t$ and substitute it into the definition of $T_2$. Evaluating at $t=0$ (where $\mu_{\pi_0} = \mu_\pi$ and $v_0 = v_t|_{t=0}$), the integral of the final term is:
    \[
        - \int_{\Sstate} \left( \int_{\Action} \frac{\dot{\mu}_{\pi_0}(s)}{\mu_{\pi_0}(s)} \langle \nabla_a Q^\pi(s, a), v_0(s, a) \rangle \, \mathrm{d}\pi(s)(a) \right) \mu_{\pi_0}(\mathrm{d}s).
    \]
    The $\mu_{\pi_0}(s)$ in the denominator cancels with the integration measure $\mu_{\pi_0}(\mathrm{d}s)$, leaving exactly the negative of the measure variation component $T_1$ defined previously:
    \[
        - \int_{\Sstate} \left( \int_{\Action} \langle \nabla_a Q^\pi(s, a), v_0(s, a) \rangle \, \mathrm{d}\pi(s)(a) \right) \dot{\mu}_{\pi_0}(\mathrm{d}s) = - T_1.
    \]
    Summing all parts, we obtain the stated result for $T_2$.
\end{proof}

\subsection{The Full Riemannian Hessian}

Before stating the full Riemannian Hessian, we introduce some necessary notations to capture the non-local dependencies of the policy variation. First, we define the \textbf{local alignment functional} $G_\pi[v]: \Sstate \to \mathbb{R}$ for a velocity field $v$:
\[
    G_\pi[v](s) := \int_{\Action} \langle \nabla_a Q^\pi(s, a), v(s, a) \rangle \, \mathrm{d}\pi(s)(a).
\]
Notice that the expectation of this functional with respect to the stationary distribution exactly recovers the first-order variation of the objective. By Proposition \ref{prop:wasserstein_grad} and the definition of our Riemannian metric, evaluating this at the initial geodesic velocity $v_0$ yields:
\[
    \langle \mu_\pi, G_\pi[v_0] \rangle = \int_{\Sstate} \int_{\Action} \langle \nabla_a Q^\pi(s, a), v_0(s, a) \rangle \, \mathrm{d}\pi(s)(a) \, \mathrm{d}\mu_\pi(s) = \langle \text{grad}_\pi J, v_0 \rangle_\pi = \dot{J}(\pi).
\]
Hence, we can define $\Psi^\pi[v]$ as the mean-zero solution to the Poisson equation driven by this local alignment:
\[
    (I - P^\pi)\Psi^\pi[v] = G_\pi[v] - \dot{J}(\pi)\mathbf{1}.
\]
Finally, to map state functions back to the state-action space, we define for any bounded function $f \in L^\infty(\Sstate)$ its associated transition projection:
\[
    Q^f(s, a) := \int_{\Sstate} f(s') \mathcal{T}(s, a, \mathrm{d}s').
\]

\begin{theorem}[Riemannian Hessian of the RL Objective]
The Hessian of the objective $J(\pi)$ at $\pi$ along the tangent vector $\xi$ (generated by $v_0$) is:
\begin{equation}
    \text{Hess}_\pi J(\xi, \xi) = \mathbb{E}_{s \sim \mu_\pi, a \sim \pi} \left[ \text{Hess}_a Q^\pi(s, a)(v_0, v_0) + \langle \nabla_a Q^{\Psi^\pi[v_0]}(s, a), v_0(s, a) \rangle + \langle \nabla_a Q^\pi(s, a), \nabla_a U_0(s, a) \rangle \right].
\end{equation}
\end{theorem}

\begin{proof}
By definition, the Riemannian Hessian is the second time derivative of the objective along a geodesic evaluated at $t=0$:
\[
    \text{Hess}_\pi J(\xi, \xi) = \left. \frac{d^2 J}{dt^2} \right|_{t=0} = T_1 + T_2.
\]

\textbf{Step 1: Cancellation of the Measure Variation Term}
From Proposition \ref{prop:T_2term}, the inner variation $T_2$ explicitly contains the negative of the measure variation $T_1$:
\[
    T_2 = \int_{\Sstate} \int_{\Action} \left[ \text{Hess}_a Q^\pi(s, a)(v_0, v_0) + \langle \nabla_a \dot{Q}^\pi, v_0 \rangle + \langle \nabla_a Q^\pi, \nabla_a U_0 \rangle \right] \mathrm{d}\pi(s)(a) \, \mathrm{d}\mu_\pi(s) - T_1.
\]
Summing $T_1$ and $T_2$, the $T_1$ terms cancel each other out completely:
\begin{equation}
    \text{Hess}_\pi J(\xi, \xi) = \mathbb{E}_{s \sim \mu_\pi, a \sim \pi} \left[ \text{Hess}_a Q^\pi(v_0, v_0) + \langle \nabla_a \dot{Q}^\pi, v_0 \rangle + \langle \nabla_a Q^\pi, \nabla_a U_0 \rangle \right].
    \label{eq:hessian_intermediate}
\end{equation}

\textbf{Step 2: Analysis of the Time Derivative $\dot{Q}^\pi$}
To finalize the expression, we must express $\dot{Q}^\pi$ purely in terms of spatial gradients. Recall the definition of the Q-function:
\[
    Q^\pi(s, a) = c(s, a) + \int_{\Sstate} V^\pi(s') \mathcal{T}(s, a, \mathrm{d}s').
\]
Differentiating with respect to time $t$ at $t=0$ and applying the gradient operator $\nabla_a$:
\begin{equation}
    \nabla_a \dot{Q}^\pi(s, a) = \nabla_a \int_{\Sstate} \dot{V}^\pi(s') \mathcal{T}(s, a, \mathrm{d}s') = \nabla_a Q^{\dot{V}^\pi}(s, a).
    \label{eq:grad_Q_dot}
\end{equation}

\textbf{Step 3: Characterizing $\dot{V}^\pi$ via the Poisson Equation}
We determine $\dot{V}^\pi$ by differentiating the Poisson equation $(I - P^\pi)V^\pi = C_\pi - J(\pi)$ with respect to time:
\[
    (I - P^\pi)\dot{V}^\pi = \dot{P}^\pi V^\pi + \dot{C}_\pi - \dot{J}(\pi).
\]
We simplify the source term $\dot{P}^\pi V^\pi + \dot{C}_\pi$. The variation of the cost is $\dot{C}_\pi(s) = \int_{\Action} c(s, a) \delta \pi(s)(\mathrm{d}a)$, and the variation of the operator is $(\dot{P}^\pi V^\pi)(s) = \int_{\Action} \left( \int_{\Sstate} V^\pi(s') \mathcal{T}(s, a, \mathrm{d}s') \right) \delta \pi(s)(\mathrm{d}a)$. Summing these and substituting the continuity equation $\delta \pi = -\nabla_a \cdot (\pi v_0)$:
\begin{align*}
    (\dot{P}^\pi V^\pi + \dot{C}_\pi)(s) &= \int_{\Action} Q^\pi(s, a) \delta \pi(s)(\mathrm{d}a) \\
    &= \int_{\Action} \langle \nabla_a Q^\pi(s, a), v_0(s, a) \rangle \, \mathrm{d}\pi(s)(a) \equiv G_\pi[v_0](s).
\end{align*}
Thus, $\dot{V}^\pi$ satisfies the Poisson equation:
\[
    (I - P^\pi)\dot{V}^\pi = G_\pi[v_0] - \dot{J}(\pi)\mathbf{1}.
\]
Let $\Psi^\pi[v_0]$ be the unique mean-zero solution to this equation. Since the solution to the Poisson equation is unique up to an additive constant, we have $\dot{V}^\pi(s) = \Psi^\pi[v_0](s) + k$. Since adding a constant to the value function does not affect its spatial gradient with respect to actions, Equation \eqref{eq:grad_Q_dot} becomes:
\[
    \nabla_a \dot{Q}^\pi(s, a) = \nabla_a Q^{\Psi^\pi[v_0]}(s, a).
\]

\textbf{Conclusion}
Substituting this result back into Equation \eqref{eq:hessian_intermediate} yields the final Riemannian Hessian formulation:
\[
    \text{Hess}_\pi J(\xi, \xi) = \mathbb{E}_{s \sim \mu_\pi, a \sim \pi} \left[ \text{Hess}_a Q^\pi(v_0, v_0) + \langle \nabla_a Q^{\Psi^\pi[v_0]}, v_0 \rangle + \langle \nabla_a Q^\pi, \nabla_a U_0 \rangle \right].
\]
\end{proof}

%
%

\begin{remark}[Quadratic Nature of the Anticipation Term]
At first glance, the term $\langle \nabla_a Q^\pi(s, a), \nabla_a U_0(s, a) \rangle$ might not appear quadratic in the velocity field $v_0$, which is expected for a Hessian form. However, the quadratic dependence is intrinsically embedded within the dynamic anticipation potential $U_0$. Recalling the definitions from Theorem \ref{th:geodesic_equations}:
\begin{enumerate}
    \item The local kinetic energy $\mathcal{E}_0(s) = \int_{\Action} \|v_0(s, a)\|^2 \, \mathrm{d}\pi(s)(a)$ is strictly quadratic in $v_0$.
    \item The kinetic value function $W_0$ is the solution to the linear Poisson equation $(I - P^\pi)W_0 = \frac{1}{2} \mathcal{E}_0 - K(0)\mathbf{1}$. Since the operator is linear and the source term depends linearly on $\mathcal{E}_0$, $W_0$ is inherently quadratic in $v_0$.
    \item The potential $U_0(s, a) = \int_{\Sstate} W_0(s') \mathcal{T}(s, a, \mathrm{d}s')$ is a linear integral projection of $W_0$, thereby preserving this quadratic nature.
\end{enumerate}
Consequently, $\nabla_a U_0$ is proportional to $\|v_0\|^2$. Taking the inner product with the velocity-independent gradient $\nabla_a Q^\pi$ ensures that the entire term constitutes a valid quadratic form representing the non-local geometric curvature.
\end{remark}

\subsubsection{Conditions for Geodesic Convexity}

Are there regimes where the objective $J(\pi)$ is geodesically convex? We identify one limiting case: the \textbf{Decoupled Dynamics}, where the environment transitions are independent of actions (i.e., $\mathcal{T}(s, a, \mathrm{d}s') = \mathcal{P}(s, \mathrm{d}s')$).
In this setting, the non-local curvature terms in our Hessian formulation vanish completely, and the optimization landscape is geodesically convex if and only if the immediate cost function $c(s, \cdot)$ is geodesically convex on $\Ptwo(\Action)$ for $\mu$-almost every $s$.

\begin{proof}[Proof of the claim]
We analyze the Riemannian Hessian of the objective function under the decoupled dynamics assumption to establish the convexity conditions.

\textbf{1. Vanishing of the Non-Local Curvature Terms.}
By assumption, the transition kernel $\mathcal{T}(s, a, \mathrm{d}s') = \mathcal{P}(s, \mathrm{d}s')$ is independent of the action $a$. Consequently, for any state function $f \in L^\infty(\Sstate)$, its transition projection $Q^f$ depends only on the state $s$:
\[
    Q^f(s, a) = \int_{\Sstate} f(s') \mathcal{P}(s, \mathrm{d}s').
\]
Because $Q^f$ is constant with respect to the action $a$, its spatial gradient with respect to actions is identically zero: $\nabla_a Q^f(s, a) = 0$.
We apply this property directly to the two non-local terms in the Hessian derived in Theorem 4.2:
\begin{itemize}
    \item The dynamic anticipation potential $U_0$ is defined as $U_0(s, a) = \int_{\Sstate} W_0(s') \mathcal{P}(s, \mathrm{d}s')$. Thus, $\nabla_a U_0(s, a) = 0$, and the third term of the Hessian $\langle \nabla_a Q^\pi, \nabla_a U_0 \rangle$ vanishes.
    \item The term involving the measure variation is driven by the potential $\Psi^\pi[v_0]$. Its projection $Q^{\Psi^\pi[v_0]}(s, a) = \int_{\Sstate} \Psi^\pi[v_0](s') \mathcal{P}(s, \mathrm{d}s')$ is also independent of $a$. Thus, $\nabla_a Q^{\Psi^\pi[v_0]}(s, a) = 0$, and the second term of the Hessian vanishes.
\end{itemize}

\textbf{2. Simplification of the Local Geometric Convexity.}
Because the transition kernel is independent of the policy $\pi$, the stationary distribution $\mu_\pi$ reduces to a fixed, policy-independent distribution $\mu$. Furthermore, the Q-function simplifies to:
\[
    Q^\pi(s, a) = c(s, a) + \int_{\Sstate} V^\pi(s') \mathcal{P}(s, \mathrm{d}s').
\]
Since the integral term is independent of $a$, the action-Hessian of the Q-function is exactly the action-Hessian of the immediate cost:
\[
    \text{Hess}_a Q^\pi(s, a) = \text{Hess}_a c(s, a).
\]

\textbf{Conclusion.}
Substituting these simplifications into the full Riemannian Hessian, we obtain a purely local form:
\[
    \text{Hess}_\pi J(\xi, \xi) = \mathbb{E}_{s \sim \mu, a \sim \pi} \left[ \text{Hess}_a c(s, a)(v_0, v_0) \right].
\]
\textit{Sufficiency ($\Leftarrow$):} If the immediate cost $a \mapsto c(s,a)$ is geodesically convex for $\mu$-almost every $s$, then by the classic second-order analysis of potential energies in Wasserstein spaces, $\text{Hess}_a c(s,a) \succeq 0$. This implies $\text{Hess}_\pi J(\xi, \xi) \ge 0$, meaning the objective $J(\pi)$ is geodesically convex.

\textit{Necessity ($\Rightarrow$):} Conversely, if $J(\pi)$ is geodesically convex, then the integral of the local Hessians is non-negative for all valid velocity fields $v_0$. If there existed a subset of states of positive $\mu$-measure where local geodesic convexity failed, we could construct a localized velocity field to make the global Hessian strictly negative, violating the assumption. Thus, local geodesic convexity of $c(s, \cdot)$ must hold $\mu$-almost everywhere.
\end{proof}

\section{Numerical experiments}\label{sec_num_analysis}
 Before detailing the environments and results, we wish to clarify the scope of these experiments. The following empirical evaluations are not intended to challenge highly tuned, state-of-the-art Reinforcement Learning baselines on standard complex benchmarks. Instead, this section is primarily addressed to the Optimal Transport and applied mathematics communities. Our objective is to provide a pedagogical proof-of-concept demonstrating how the first-order geometric objects derived in Section \ref{sec:wass_geom_optim}—specifically the Wasserstein policy gradient and the resulting Riemannian pullback metric—can be translated into tractable numerical schemes. 
\subsection{Description of environments}
For the following experiments, we will use three models. The first one is a toy environment in a scalar state space; it will be called the scalar stochastic non-linear regulator. Then, we consider the classical inverted pendulum problem where the control corresponds to the torque applied. Finally, we consider a high-dimensional environment of coupled oscillators.

\subsubsection{A scalar stochastic non-linear regulator}
This environment is defined as a scalar, stochastic, non-linear control task.

\paragraph{State and Action Spaces}
The state $s_t \in \mathcal{S} \subset \mathbb{R}$ and action $a_t \in \mathcal{A} \subset \mathbb{R}$ are bounded scalars:
\begin{equation}
    \mathcal{S} = [s_{\min}, s_{\max}], \quad \mathcal{A} = [a_{\min}, a_{\max}].
\end{equation}

\paragraph{Dynamics}
The system evolves according to a non-linear transition function subject to additive Gaussian noise:
\begin{equation}
    s_{t+1} = \text{clip}\left( \alpha s_t + \beta \sin(s_t) + \delta a_t + \xi_t, s_{\min}, s_{\max} \right),
\end{equation}
where $\xi_t \sim \mathcal{N}(0, \sigma^2)$ is independent Gaussian noise. 
\textit{Remark: The dynamics include a clamping operation to ensure the state $s_{t+1}$ remains within the bounded domain $\mathcal{S}$ at every step.}

\paragraph{Cost Function}
The objective is to minimize the quadratic cost:
\begin{equation}
    c(s_t, a_t) = s_t^2 + \lambda a_t^2.
\end{equation}

\begin{table}[h]
\centering
\begin{tabular}{|l|c|c|}
\hline
\textbf{Parameter} & \textbf{Symbol} & \textbf{Value} \\
\hline
State Coefficient & $\alpha$ & 0.5 \\
Non-linear Coefficient & $\beta$ & 0.1 \\
Control Coefficient & $\delta$ & 0.2 \\
Noise Std. Dev. & $\sigma$ & 0.2 \\
Action Penalty & $\lambda$ & 0.1 \\
State Bounds & $[s_{\min}, s_{\max}]$ & $[-3.0, 3.0]$ \\
Action Bounds & $[a_{\min}, a_{\max}]$ & $[-5.0, 5.0]$ \\
\hline
\end{tabular}
\caption{Parameters for the Exact Problem}
\end{table}

\subsubsection{Inverted Pendulum}
The Inverted Pendulum is a classic under-actuated mechanical system.

\paragraph{State and Action Spaces}
The state vector is $x_t = [\theta_t, \dot{\theta}_t]^\top$, where $\theta_t \in [-\pi, \pi]$ represents the angle from the upright position and $\dot{\theta}_t \in [-\dot{\theta}_{\max}, \dot{\theta}_{\max}]$ is the angular velocity. The control input $u_t \in [-u_{\max}, u_{\max}]$ is the torque applied at the pivot.

\paragraph{Dynamics}
The system is simulated using semi-implicit Euler integration with time step $\Delta t$. The angular acceleration includes gravity and control torque:
\begin{align}
    \dot{\theta}_{t+1} &= \text{clip}\left( \dot{\theta}_t + \Delta t \left[ -\frac{3g}{2l} \sin(\theta_t + \pi) + \frac{3}{ml^2} u_t \right], -\dot{\theta}_{\max}, \dot{\theta}_{\max} \right) \\
    \theta_{t+1} &= \text{wrap}\left( \theta_t + \Delta t \cdot \dot{\theta}_{t+1} \right)
\end{align}
where $\text{wrap}(\cdot)$ normalizes the angle to $[-\pi, \pi]$. 
\textit{Remark: Explicit clamping is applied to the angular velocity $\dot{\theta}$ to enforce physical limits, while the angle $\theta$ is wrapped to remain within $[-\pi, \pi]$.}

\paragraph{Cost Function}
The cost penalizes deviations from the upright equilibrium and control effort:
\begin{equation}
    c(\mathbf{x}_t, u_t) = \theta_t^2 + w_{\dot{\theta}} \dot{\theta}_t^2 + w_{u} u_t^2.
\end{equation}

\begin{table}[h]
\centering
\begin{tabular}{|l|c|c|}
\hline
\textbf{Parameter} & \textbf{Symbol} & \textbf{Value} \\
\hline
Gravity & $g$ & $10.0$ m/s$^2$ \\
Mass & $m$ & $1.0$ kg \\
Length & $l$ & $1.0$ m \\
Time Step & $\Delta t$ & $0.05$ s \\
Max Velocity & $\dot{\theta}_{\max}$ & $8.0$ rad/s \\
Max Torque & $u_{\max}$ & $2.0$ N$\cdot$m \\
Weights & $w_{\dot{\theta}}, w_{u}$ & $0.1, 0.01$ \\
\hline
\end{tabular}
\caption{Parameters for the Inverted Pendulum}
\end{table}

\subsubsection{High-Dimensional Coupled Oscillators}
This environment models a chain of $N$ masses connected by springs and subject to non-linear Duffing forces and inter-particle repulsion.

\paragraph{State and Action Spaces}
The state is $\mathbf{s}_t = [x_t, v_t]^\top \in \mathbb{R}^{2N}$, where $x_t, v_t \in \mathbb{R}^N$ are the positions and velocities of the $N$ masses. The system is fully actuated with control $u_t \in \mathbb{R}^N$. Constraints are applied such that $\|x_t\|_\infty \le x_{\max}$, $\|v_t\|_\infty \le v_{\max}$, and $\|u_t\|_\infty \le u_{\max}$.

\paragraph{Dynamics}
The dynamics for the $i$-th mass ($i=1\dots N$) are governed by nearest-neighbor interactions. We define boundary conditions $x_{0} = x_{N+1} = 0$. The acceleration $\ddot{x}_{i}$ is given by:
\begin{equation}
    \ddot{x}_i = F_{spring} + F_{repulsion} + F_{duffing} + F_{damp} + u_i
\end{equation}
where:
\begin{itemize}
    \item $F_{spring} = k(x_{i+1} - 2x_i + x_{i-1})$
    \item $F_{repulsion} = A \exp\left(-\frac{|x_i - x_{i-1}|}{w}\right) - A \exp\left(-\frac{|x_{i+1} - x_i|}{w}\right)$
    \item  $F_{duffing} = -\alpha x_i^3$
    \item  $F_{damp} = -\beta \dot{x}_i$
    \item $u_i$ is the control actuation.
\end{itemize}

The discrete update rule follows Euler integration:
\begin{align}
    v_{t+1} &= \text{clip}(v_t + \Delta t \cdot \ddot{x}_t, -v_{\max}, v_{\max}) \\
    x_{t+1} &= \text{clip}(x_t + \Delta t \cdot v_{t+1}, -x_{\max}, x_{\max})
\end{align}

\textit{Remark: The simulation strictly clips both the positions $\mathbf{x}$ and velocities $\mathbf{v}$ to their respective maximum values at every time step to prevent instability and enforce state space bounds.}

\paragraph{Cost Function}
The objective is to stabilize the chain at the origin:
\begin{equation}
    c(s_t, u_t) = \sum_{i=1}^N \left( x_{i,t}^2 + w_v v_{i,t}^2 + w_u u_{i,t}^2 \right).
\end{equation}

\begin{table}[h]
\centering
\begin{tabular}{|l|c|c|}
\hline
\textbf{Parameter} & \textbf{Symbol} & \textbf{Value} \\
\hline
Number of Masses & $N$ & 5 \\
Spring Stiffness & $k$ & $1.0$ \\
Duffing Coeff. & $\alpha$ & $1.0$ \\
Damping & $\beta$ & $0.1$ \\
Repulsion Amp/Width & $A, w$ & $10.0, 1.0$ \\
Time Step & $\Delta t$ & $0.05$ s \\
Constraints & $x_{\max}, v_{\max}, u_{\max}$ & $5.0$ \\
Weights & $w_v, w_u$ & $0.1, 0.001$ \\
\hline
\end{tabular}
\caption{Parameters for High-Dimensional Oscillators}

\end{table}

\subsection{Description of the numerical methods}

For the first algorithm, we apply directly the methodology developed in section \ref{sec:wass_geom_optim}. Here the policy has the form of a convex combination of Diracs called particles:

$$
\pi(s) = \frac{1}{M} \sum_{i=1}^M \delta_{a_i(s)} 
$$
where the $(a_i)$ are the parameters to optimize. The pseudo-code is given in Algorithm 1.

\begin{algorithm}[H]
\caption{Grid-based Policy Iteration with Particles}
\begin{algorithmic}[1]
\Require State Grid $\mathcal{S} = \{s_1, \dots, s_N\}$, Number of particles $M$, Smoothing parameter $\sigma$, Learning rate $\eta$.
\Ensure Value Function $V$ and Policy $\pi$ (defined by particles).

\State \textbf{Initialization:}
\State Randomly initialize action particles $\theta_{s,i} \in \mathcal{A}$ for each state $s \in \mathcal{S}$ and $i \in \{1, \dots, M\}$.

\For{iteration $k = 1, \dots, K$}
    \State \textbf{1. Transition and Cost Calculation (Gaussian Smoothing)}
    \For{each state $s \in \mathcal{S}$ and particle $i \in \{1, \dots, M\}$}
        \State Simulate physical next state: $s'_{phy} = f_{dyn}(s, \theta_{s,i})$
        \State Compute immediate cost: $c_{s,i} = \text{Cost}(s, \theta_{s,i})$
        \State Compute transition probabilities (Kernel):
        \State $P(s' | s, i) \propto \exp\left(-\frac{\|s' - s'_{phy}\|^2}{2\sigma^2}\right)$ \Comment{Normalized over $s' \in \mathcal{S}$ to sum to 1}
    \EndFor

    \State \textbf{2. Aggregation (Current Policy)}
    \State Mean transition matrix: $P^\pi_{s,s'} = \frac{1}{M} \sum_{i=1}^M P(s' | s, i)$
    \State Mean cost vector: $C^\pi_s = \frac{1}{M} \sum_{i=1}^M c_{s,i}$

    \State \textbf{3. Stationary Distribution and Average Cost}
    \State Compute stationary distribution $\mu \in \mathcal{P}(\mathcal{S})$ satisfying $\mu^\top P^\pi = \mu^\top$ with $\sum_{s \in \mathcal{S}} \mu_s = 1$.
    \State Global average cost: $J(\pi) = \sum_{s \in \mathcal{S}} \mu_s C^\pi_s$.
    \State Centered cost vector: $\bar{C}^\pi_s = C^\pi_s - J(\pi)$.

    \State \textbf{4. Policy Evaluation (Poisson Equation)}
    \State Solve the centered Poisson equation for the bias function $V$:
    \State $(I - P^\pi) V = \bar{C}^\pi \quad \text{subject to } \sum_{s \in \mathcal{S}} \mu_s V_s = 0$.

    \State \textbf{5. Policy Improvement (Wasserstein Gradient Flow)}
    \For{each state $s \in \mathcal{S}$ and particle $i \in \{1, \dots, M\}$}
        \State Evaluate state-action value: $Q(s, \theta_{s,i}) = c_{s,i} + \sum_{s' \in \mathcal{S}} P(s' | s, i) V_{s'}$
        \State Update particle along the action gradient:
        \State $\theta_{s,i} \leftarrow \theta_{s,i} - \eta \, \nabla_a Q(s, \theta_{s,i})$
        \State Project $\theta_{s,i}$ within admissible action bounds $\mathcal{A}$.
    \EndFor
\EndFor
\State \Return $V$ and optimized particles $\theta$.
\end{algorithmic}
\end{algorithm}
The method described in Algorithm 1 works only if the state space is low-dimensional. Indeed, it requires building the probability transition matrix $P(s' \mid s, i)$, which is very computationally expensive to manipulate. To remedy this problem, we represent the policy with an optimized neural network and use an ergodic approximation of $J(\pi)$, i.e.:

$$
J(\pi) \simeq \frac{1}{H} \sum_{t=1}^H c(s_t,a_t) 
$$
where $(s_t,a_t)$ is a trajectory following the policy $\pi$ in the sense that the law of $a_t$ is $\pi(s_t)$. By automatic differentiation, it is possible to compute its gradient if the environment is \textbf{differentiable}. For the following computations, we assume a deterministic policy in the sense that:

$$
\pi_\theta(s) = \delta(a - \mu_\theta(s)).
$$

\begin{algorithm}[H]
\caption{Trajectory Optimization (Adam or Natural-Gradient)}
\begin{algorithmic}[1]
\Require Differentiable dynamics $s_{t+1} = f(s_t, a_t)$, Cost function $c(s, a)$.
\Require Optimization Method $M \in \{\text{Adam}, \text{NG}\}$.
\Require Horizon $H$, Batch size $B$, Learning rate $\eta$.
\Ensure Policy network parameters $\theta$.

\State \textbf{Initialization:} Random weights $\theta$, Start Buffer $\mathcal{B}$.

\For{iteration $k = 1, \dots, K$}
    \State \textbf{1. Start State Sampling}
    \State Sample a batch of start states $s_0 \sim \mathcal{B}$ of size $B$.
    \State Initialize cumulative loss $J = 0$.
    \State Set current state $s = s_0$.
    
    \State \textbf{2. Trajectory Simulation (Forward Pass)}
    \For{time step $t = 0, \dots, H-1$}
        \State Compute action: $a_t = \mu_\theta(s_t)$
        \State Compute immediate cost: $J \leftarrow J + \frac{c(s_t, a_t)}{H} $
        \State \textbf{Dynamics:} Compute next state via physics:
        \State $s_{t+1} = f(s_t, a_t)$ \Comment{Differentiable operation}
        \State $s_t \leftarrow s_{t+1}$
    \EndFor
    \State Fill the buffer $B_f$ of all the states $s_t$ encountered to approximate the stationary distribution.

    \State \textbf{3. Gradient Computation (Backward Pass)}
    \State Compute gradient: $\text{flat\_grad} = \nabla_\theta J$.

    \State \textbf{4. Parameter Update}
    \If{$M = \text{Adam}$}
        \State Update $\theta$ using standard Adam rule with learning rate $\eta$.
    \ElsIf{$M = \text{NG}$}
        \State Define Metric-vector product function $f_{Mx}(v) = (M(\theta) + \lambda_{\text{damp}} I) \cdot v$ with:
        \State $M(\theta) \simeq \frac{1}{|B_f|} \sum_{s \in B_f} (\nabla_\theta \mu_\theta(s))^\top \nabla_\theta \mu_\theta(s)$ using the buffer $B_f$.
        \State \textbf{Solve} via Conjugate Gradient:
        $$ (M(\theta) + \lambda_{\text{damp}} I) v = \text{flat\_grad} $$
        \State (Solved without explicit matrix storage via $f_{Mx}$ to save memory).
        \State Update parameters: $\theta \leftarrow \theta - \eta v$.
        \State Empty $B_f$
    \EndIf

    \State \textbf{5. Exploration Update}
    \State Add visited states $s_t$ to $\mathcal{B}$ to diversify start states.
\EndFor
\State \Return Optimized policy $\mu_\theta$.
\end{algorithmic}
\end{algorithm}

It is crucial to differentiate the optimization performed by standard algorithms (such as Adam or SGD) from the theoretical gradient flow described in our framework. Standard methods implicitly operate within the parameter space $\Theta$ equipped with a flat (Euclidean) metric. In contrast, our theory relies on a geometry defined directly on the space of policies (function space). 

To ensure consistency between the algorithm and the theory, we must equip the parameter space with the \textit{pullback metric} induced by the policy metric $g_\pi$ defined in Equation~\eqref{eq:def_metric}. To do so, we need Lemma \ref{lemma:transport_equivalence}.

\begin{lemma}[Equivalence of Parametric Variation and Transport]
\label{lemma:transport_equivalence}
Let $\mu_\theta: \mathcal{S} \to \mathcal{A}$ be a $C^1$-differentiable policy parameterized by $\theta \in \mathbb{R}^d$. Consider the deterministic policy defined by the Dirac distribution $\pi_\theta(s) = \delta_{\mu_\theta(s)}$.

Let $\delta \theta \in \mathbb{R}^d$ be an infinitesimal variation of the parameters. The induced variation of the policy, denoted $\delta \pi$, satisfies the continuity equation in the sense of distributions:
\begin{equation}
    \delta \pi(s) + \nabla_a \cdot (\pi_\theta(s) v(s)) = 0,
\end{equation}
where the velocity vector $v(s) \in \mathbb{R}^{d_a}$ is given by the kinematic projection of the parameter update:
\begin{equation}
    v(s) = \nabla_\theta \mu_\theta(s) \delta \theta.
\end{equation}
\end{lemma}

\begin{proof}
We proceed using the method of test functions. Let $\phi \in C_c^\infty(\mathcal{A})$ be a smooth test function with compact support. We treat the variation $\delta \pi(s)$ as a distribution acting on $\phi$ via the duality pairing $\langle \cdot, \cdot \rangle$.

\textbf{1. Parametric Variation (Left-Hand Side):}
The variation of the policy expectation with respect to the parameters is computed using the chain rule. Since $\pi_\theta(s)$ is a Dirac mass centered at $\mu_\theta(s)$, the pairing is simply the evaluation of the test function at $\mu_\theta(s)$:
\begin{align}
    \langle \delta \pi(s), \phi \rangle 
    &= \delta \left( \int_{\mathcal{A}} \phi(a) \, \mathrm{d}\delta_{\mu_\theta(s)}(a) \right) \nonumber \\
    &= \delta \left( \phi(\mu_\theta(s)) \right) \nonumber \\
    &= \nabla_a \phi(\mu_\theta(s))^\top \left( \nabla_\theta \mu_\theta(s) \delta \theta \right). \label{eq:proof_param}
\end{align}

\textbf{2. Transport Variation (Right-Hand Side):}
We now consider the distribution defined by the divergence term $-\nabla_a \cdot (\pi_\theta(s) v(s))$. By the definition of the distributional derivative, the action of the divergence is transferred to the test function:
\begin{align}
    \langle -\nabla_a \cdot (\pi_\theta(s) v(s)), \phi \rangle 
    &= \int_{\mathcal{A}} v(s)^\top \nabla_a \phi(a) \, \mathrm{d}\delta_{\mu_\theta(s)}(a) \nonumber \\
    &= v(s)^\top \nabla_a \phi(\mu_\theta(s)). \label{eq:proof_transport}
\end{align}

\textbf{Conclusion:}
Comparing \eqref{eq:proof_param} and \eqref{eq:proof_transport}, both distributions coincide on $C_c^\infty(\mathcal{A})$ if and only if:
$$
v(s)^\top \nabla_a \phi(\mu_\theta(s)) = \left( \nabla_\theta \mu_\theta(s) \delta \theta \right)^\top \nabla_a \phi(\mu_\theta(s)).
$$
Since this holds for any gradient $\nabla_a \phi$, we identify the velocity field as:
$$
v(s) = \nabla_\theta \mu_\theta(s) \delta \theta.
$$
\end{proof}

Now let $\delta\theta_1, \delta\theta_2 \in T_\theta \Theta$ be two infinitesimal variations of the parameters, and let $M(\theta)$ denote the resulting pullback metric tensor. By requiring isometry between the parameter updates and the function variations, we obtain:

\begin{align*}
    \langle \delta\theta_1, \delta\theta_2 \rangle_{M(\theta)}
    &= g_\pi(\nabla_\theta \mu_\theta \cdot \delta\theta_1, \nabla_\theta \mu_\theta \cdot \delta\theta_2) \\
    &= \int_S \left( \nabla_\theta \mu_\theta(s) \delta\theta_1 \right)^T \left( \nabla_\theta \mu_\theta(s) \delta\theta_2 \right) \, d\mu^\pi(s) \\
    &= \delta\theta_1^T \left[ \int_S \left(\nabla_\theta \mu_\theta(s)\right)^T \nabla_\theta \mu_\theta(s) \, d\mu^\pi(s) \right] \delta\theta_2.
\end{align*}

Consequently, the consistent metric $M(\theta)$ on the parameter space is identified as the expected Gram matrix of the policy Jacobians:
\begin{equation}
    M(\theta) = \int_{\Sstate} \left(\nabla_\theta \mu_\theta(s)\right)^\top \nabla_\theta \mu_\theta(s) \, \mathrm{d}\mu_\pi(s).
\end{equation}
This result establishes a fundamental theoretical insight: while the classical Fisher Information Matrix diverges to infinity for deterministic policies (due to the singularity of Dirac distributions under the Kullback-Leibler divergence), the Wasserstein geometry remains smooth and non-degenerate. Lemma \ref{lemma:transport_equivalence} reveals that the Wasserstein pullback metric provides the natural, mathematically rigorous geometric foundation for Natural Gradient descent in deterministic actor-critic architectures.

So far, we have only given approaches where we know everything about the environment; this corresponds to an optimal control problem formulation. For real reinforcement learning applications, we only require information on the cost function $c$. To deal with such constraints, we introduce a neural world model whose task is to learn the environment dynamics. Then, as in Algorithm 2, we compute the gradient of the cost through a trajectory given by the world model, which is differentiable. The pseudo-code is detailed in Algorithm 3, focusing on the Natural Gradient optimizer to align with our geometric framework.

\begin{algorithm}[H]
\caption{Model-Based RL with Differentiable World Model (VarPro)}
\begin{algorithmic}[1]
\Require Real Environment $f_{env}$, Cost function $c(s, a)$.
\Require Policy $\mu_\theta$ (Neural Net), World Model $M_\phi$ (Body $\phi$, Head $W$).
\Require Replay Buffer $\mathcal{D}$, Horizon $H$, Batch sizes $B_{pol}, B_{wm}$.

\State \textbf{Initialization:}
\State Initialize $\mu_\theta$ and $M_\phi$ weights randomly.
\State Initialize $\mathcal{D}$ with random trajectories from $f_{env}$.

\For{iteration $k = 1, \dots, K$}
    \State \textbf{1. Data Collection (Real World)}
    \State Sample start state $s_0$.
    \For{$t = 0 \dots H_{env}$}
        \State $a_t = \mu_\theta(s_t)$
        \State $s_{t+1} = f_{env}(s_t, a_t)$ \Comment{Real Dynamics}
        \State Store transition $(s_t, a_t, s_{t+1})$ in $\mathcal{D}$.
    \EndFor

    \State \textbf{2. Train World Model (VarPro / Analytic Head)}
    \For{$j = 1 \dots N_{wm\_updates}$}
        \State Sample batch $(s, a, s') \sim \mathcal{D}$.
        \State Compute targets (e.g., velocities): $y = (s' - s) / \Delta t$.
        \State Forward Body: $\Phi = \text{Body}_\phi(s, a)$.
        \State \textbf{Analytic Step:} Solve linear system for Head weights $W$:
        \State $W^* = (\Phi^T \Phi + \lambda I)^{-1} \Phi^T y$ \Comment{Ridge Regression}
        \State Update Head weights: $W \leftarrow W^*$.
        \State \textbf{Gradient Step:} Compute Loss using optimized weights:
        \State $L_{WM} = || \Phi W^* - y ||^2$
        \State Update Body: $\phi \leftarrow \phi - \eta_{wm} \nabla_\phi L_{WM}$.
    \EndFor

    \State \textbf{3. Policy Optimization (Imagination)}
    \State Sample start states $s_0$ (from buffer or random).
    \State Initialize cumulative loss $J = 0$.
    \For{$t = 0 \dots H$}
        \State $a_t = \mu_\theta(s_t)$
        \State $J \leftarrow J + \frac{c(s_t, a_t)}{H}$
        \State \textbf{Predicted Dynamics:}
        \State $\dot{s} = \text{Head}(\text{Body}_\phi(s_t, a_t))$
        \State $s_{t+1} = s_t + \dot{s} \cdot \Delta t$ \Comment{Differentiable Simulation}
    \EndFor
    \State Compute Gradient w.r.t Policy: $\nabla_\theta J$ \Comment{Flows through learned WM}
    \State Update Policy using the Natural Gradient method 
\EndFor
\State \Return $\mu_\theta$
\end{algorithmic}
\end{algorithm}

\subsection{Results}

\subsubsection{The scalar stochastic non-linear regulator}
We begin by analyzing the performance of Algorithm 1 on the scalar stochastic non-linear regulator. The learning process, illustrated in Figure \ref{fig:exact_learning}, demonstrates rapid convergence; the average cost (left) drops significantly within the first 10 iterations before stabilizing, indicating the successful identification of an optimal policy. 

Figure \ref{fig:exact_visu} details the resulting Value Function and Policy:
\begin{itemize}
    \item The converged Value Function $V^\pi$ (left) exhibits a distinct convex, parabolic shape.
    \item The Policy (right) remains approximately linear ($a \approx -Ks$) in the central region, consistent with the expected optimal control for linearized dynamics. 
\end{itemize}

\begin{figure}[h]
    \centering
    \begin{minipage}{0.45\textwidth}
        \centering
        \includegraphics[width=\linewidth]{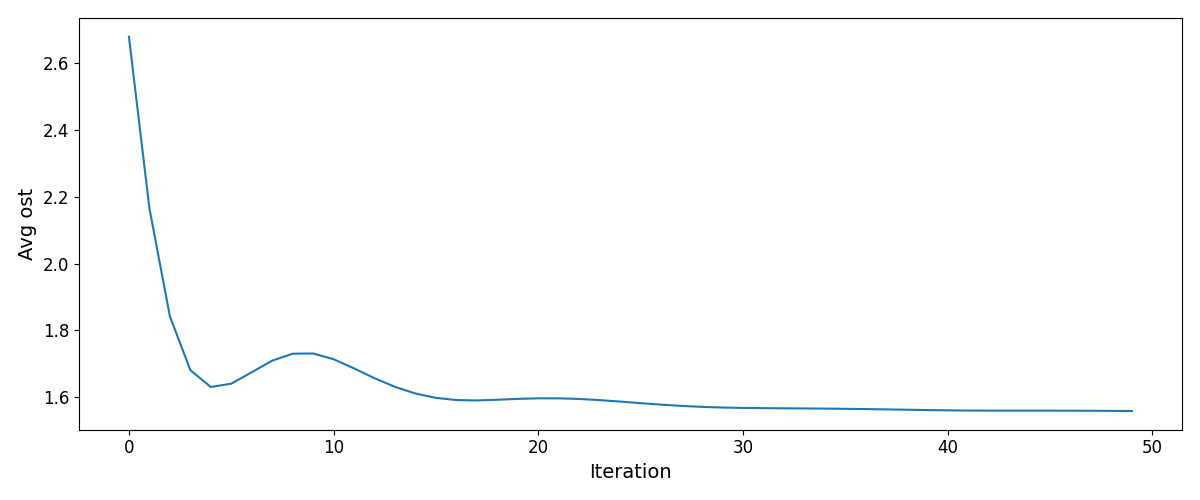}
        \caption{Convergence of the Average Cost during Policy Iteration.}
        \label{fig:exact_learning}
    \end{minipage}\hfill
    \begin{minipage}{0.5\textwidth}
        \centering
        \includegraphics[width=\linewidth]{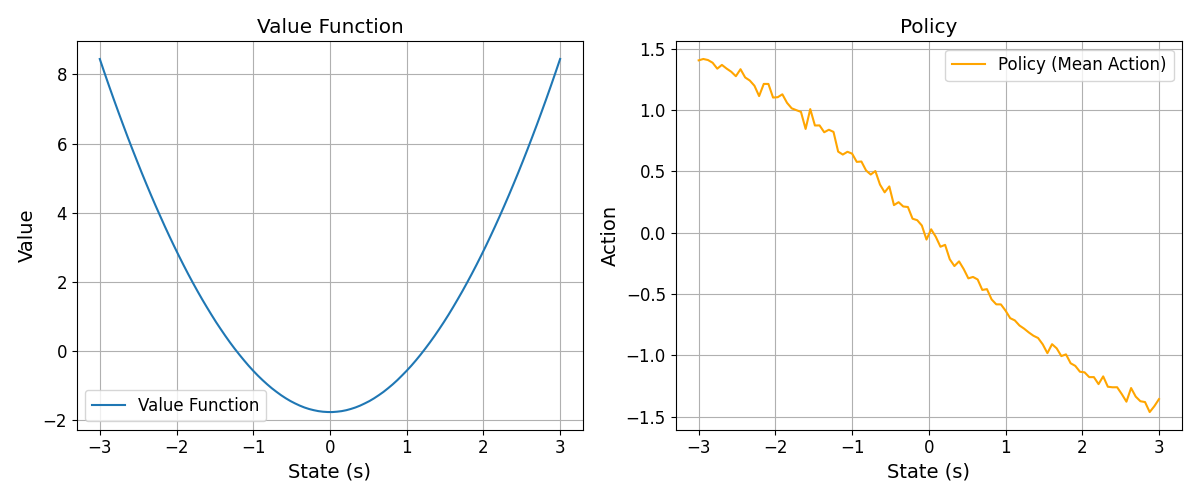}
        \caption{Learned Value Function (Parabolic) and Policy (Saturated Linear).}
        \label{fig:exact_visu}
    \end{minipage}
\end{figure}

To validate the policy's robustness, Figure \ref{fig:exact_traj} presents a simulated trajectory starting from $s_0 = 2.5$. The ``clean'' trajectory (orange dashed) confirms asymptotic stability toward the target state. The ``noisy'' trajectory (blue) further highlights the agent's resilience: despite significant Gaussian perturbations ($\sigma=0.2$), the agent actively corrects deviations (see green action curve) to maintain the state near equilibrium.

\begin{figure}[h]
    \centering
    \includegraphics[width=0.8\textwidth]{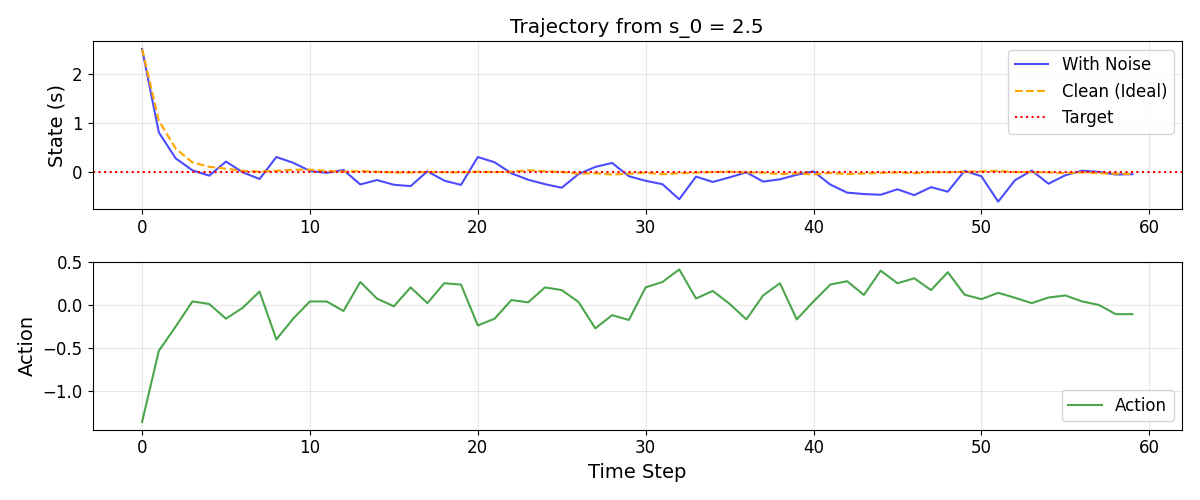}
    \caption{Simulated trajectory. Top: State evolution with and without noise. Bottom: Control actions applied by the agent.}
    \label{fig:exact_traj}
\end{figure}

\subsubsection{Inverted Pendulum}

We evaluate the Inverted Pendulum task using two distinct approaches: a discretized grid policy and a Neural Network Policy (trained with either a differentiable known environment or a learned world model).

Figure \ref{fig:pend_grid} visualizes the global solution derived by the grid agent. On the left, the value function heatmap clearly depicts a low-cost valley corresponding to the upright equilibrium ($\theta=0, \dot{\theta}=0$), forming a basin of attraction where the cost increases as the state diverges from the target. The policy heatmap (right) reveals the non-linear switching logic required to pump energy and stabilize the pendulum. Notably, the sharp transition between maximum positive (red) and maximum negative (blue) torque regions indicates a ``bang-bang'' control strategy during the energy pumping phase. The resulting swing-up behavior is confirmed by the trajectory in Figure \ref{fig:pend_grid_traj}, where the system successfully reaches the upright position in approximately 1.5 seconds.

\begin{figure}[H]
    \centering
    \includegraphics[width=0.8\textwidth]{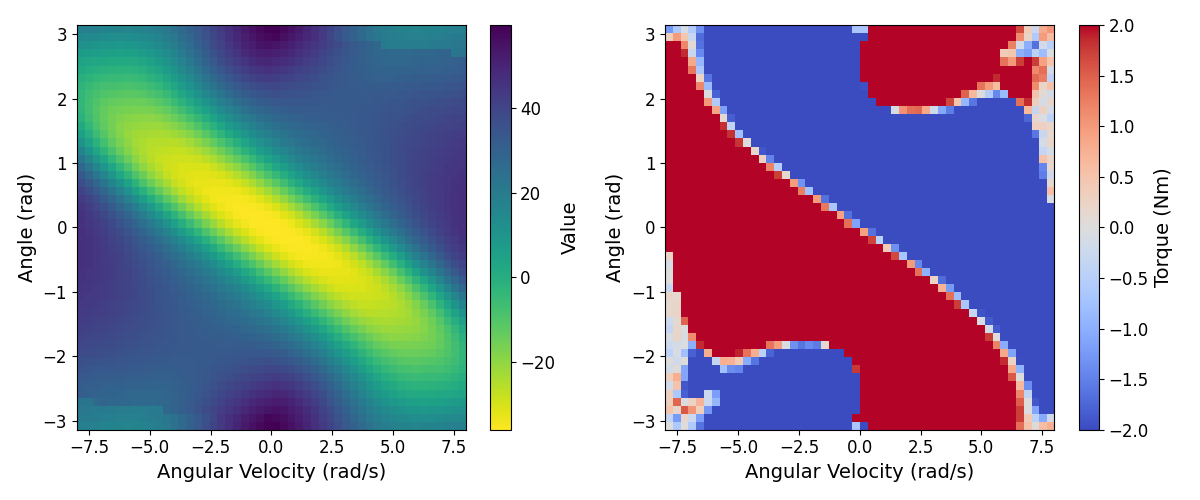}
    \caption{Grid-based results. Left: Value Function heatmap. Right: Policy mean action heatmap. }
    \label{fig:pend_grid}
\end{figure}

\begin{figure}[H]
    \centering
    \begin{minipage}{0.48\textwidth}
        \centering
        \includegraphics[width=\linewidth]{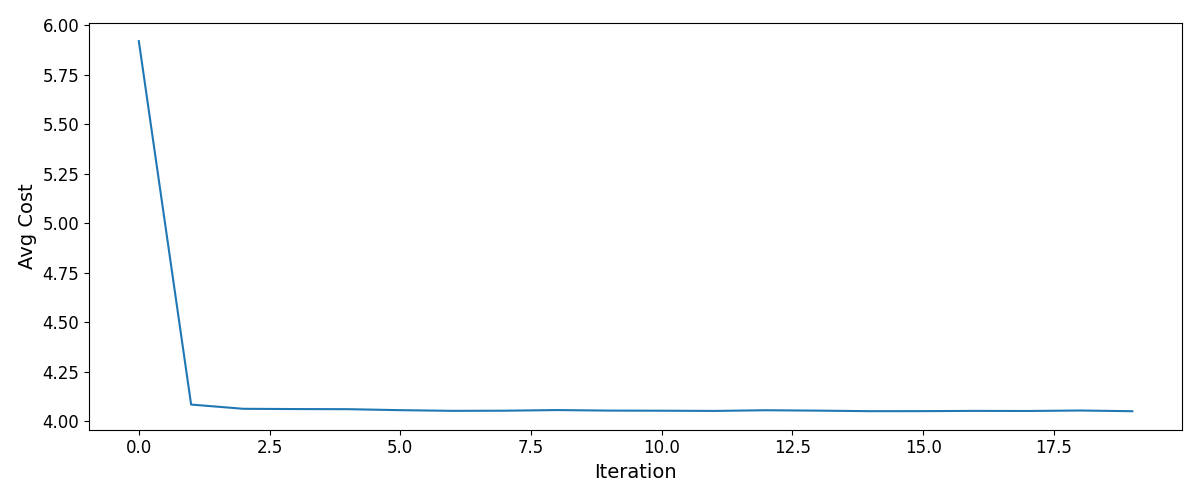}
        \caption{Convergence of Average Cost (Grid).}
    \end{minipage}\hfill
    \begin{minipage}{0.48\textwidth}
        \centering
        \includegraphics[width=\linewidth]{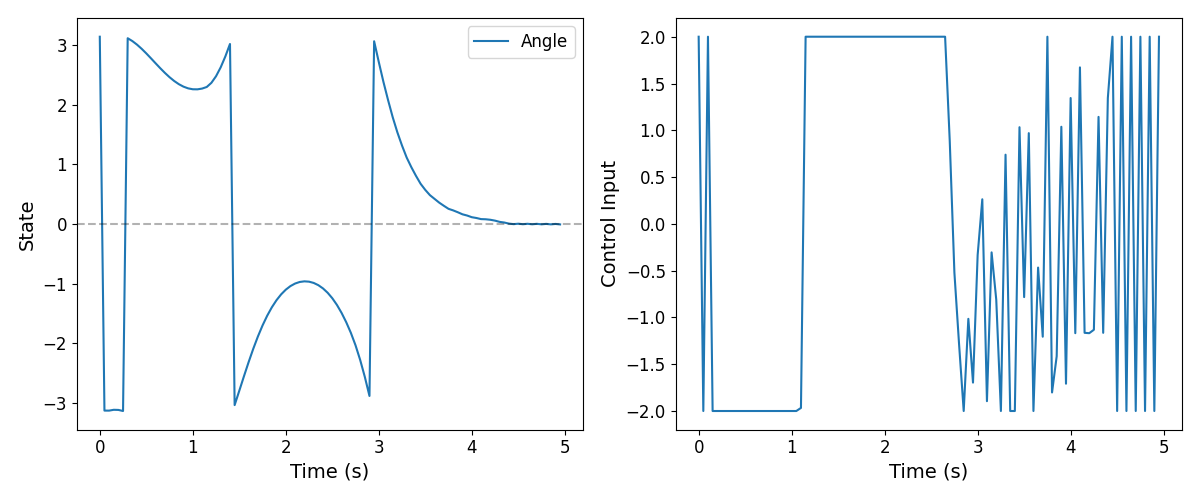}
        \caption{Swing-up trajectory (Grid Agent).}
        \label{fig:pend_grid_traj}
    \end{minipage}
\end{figure}

Next, we compare the performance of gradient-based optimization using ground-truth physics versus learned physics. As neural network algorithms are stochastic, we plot the mean and variance of the loss for ten training episodes. Represented trajectories correspond to the training episode where the loss at the last iteration is the lowest.

\paragraph{Direct Differentiable Physics (Baseline)}
In this setting, gradients are backpropagated directly through the analytic equations of motion. In contrast to the grid agent's sharp transitions, the neural policy learns a smooth, continuous control law (Fig. \ref{fig:pend_nn_adam_res}-\ref{fig:pend_nn_ng_res}). The associated trajectory demonstrates an efficient swing-up from the bottom ($\theta = \pi$), handling angle wrapping boundaries and stabilizing perfectly at zero. Regarding convergence, Figure \ref{fig:pend_nn_adam_loss}-\ref{fig:pend_nn_ng_loss} shows the loss decreasing monotonically and stabilizing, indicating that gradient descent successfully navigates the non-convex landscape of the swing-up problem. Note also that the Natural Gradient optimizer seems to give similar performance with more variance, compared to the direct Adam strategy.

\begin{figure}[H]
    \centering
    \begin{minipage}{0.48\textwidth}
        \centering
        \includegraphics[width=\linewidth]{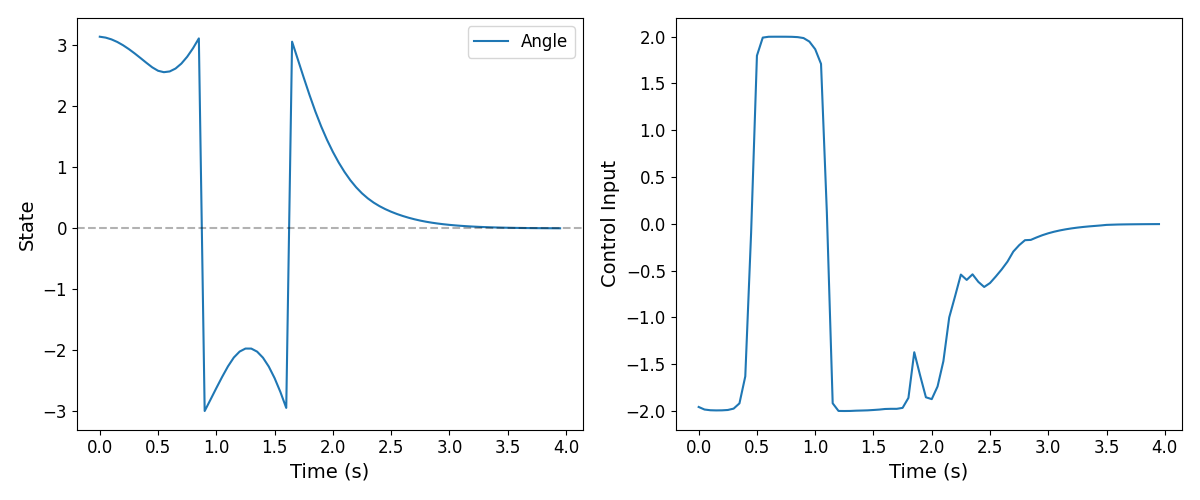}
        \caption{Direct Diff. Physics: Trajectory and Control.}
        \label{fig:pend_nn_adam_res}
    \end{minipage}\hfill
    \begin{minipage}{0.48\textwidth}
        \centering
        \includegraphics[width=\linewidth]{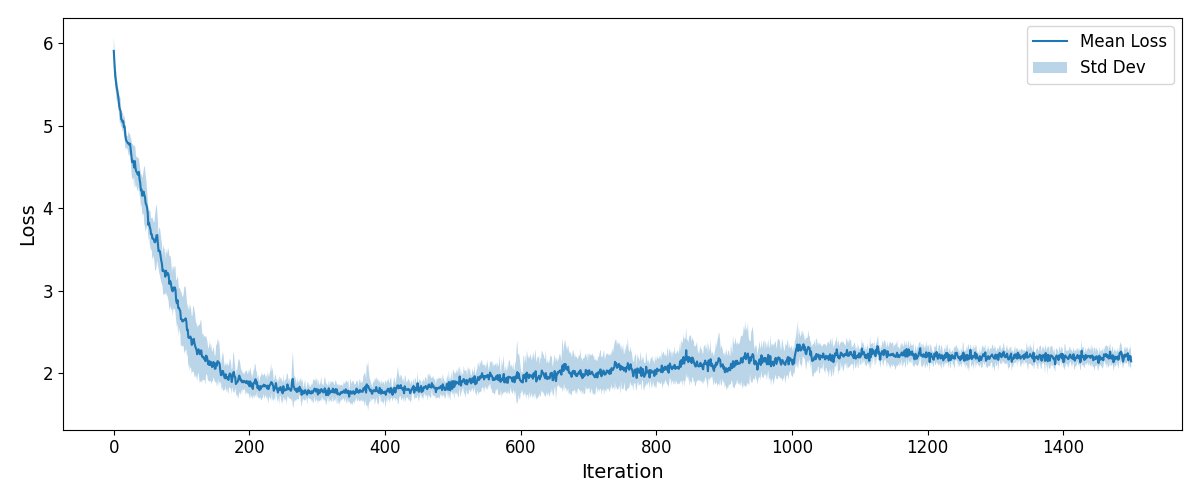}
        \caption{Direct Diff. Physics: Training Loss (Adam).}
        \label{fig:pend_nn_adam_loss}
    \end{minipage}
\end{figure}

\begin{figure}[H]
    \centering
    \begin{minipage}{0.48\textwidth}
        \centering
        \includegraphics[width=\linewidth]{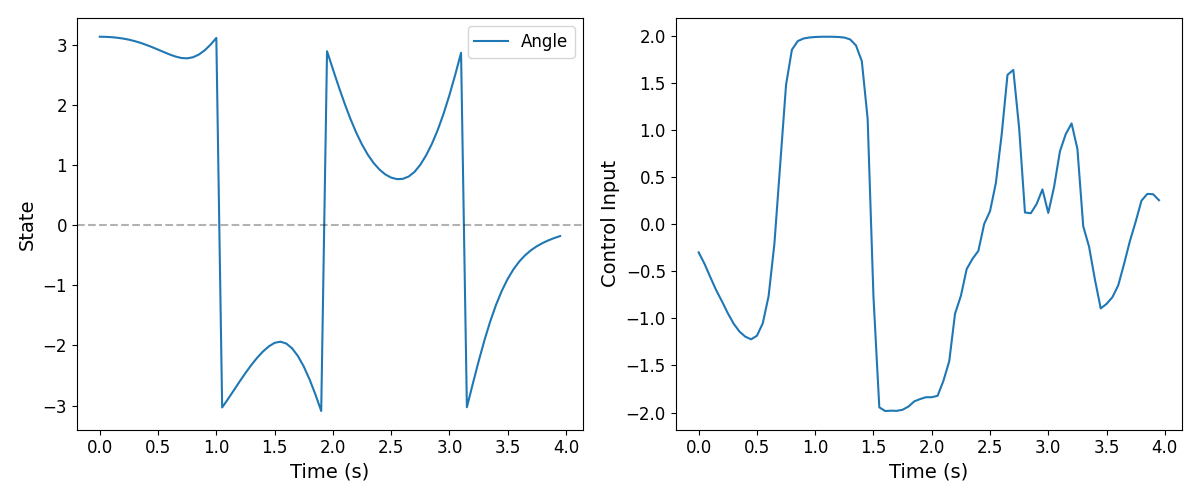}
        \caption{Direct Diff. Physics: Trajectory and Control.}
        \label{fig:pend_nn_ng_res}
    \end{minipage}\hfill
    \begin{minipage}{0.48\textwidth}
        \centering
        \includegraphics[width=\linewidth]{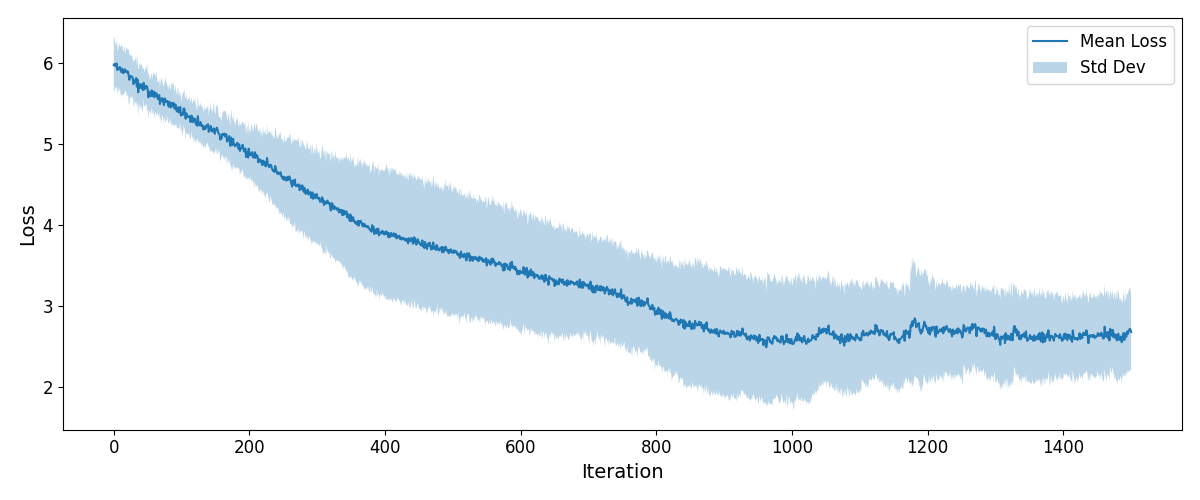}
        \caption{Direct Diff. Physics: Training Loss (Natural Gradient).}
        \label{fig:pend_nn_ng_loss}
    \end{minipage}
\end{figure}

\paragraph{Learned World Model (Natural Gradient optimizer)}
Here, the agent assumes no prior knowledge of the dynamics, instead learning a neural approximation $s_{t+1} \approx f_\phi(s_t, a_t)$. As shown in Figure \ref{fig:pend_wm_loss} (right), the \textit{World Model Loss} decreases rapidly, indicating the model successfully captures non-linear pendulum dynamics such as gravity and inertia. Simultaneously, the \textit{Policy Loss} (left) decreases as the planner optimizes actions via the learned model. The resulting trajectory (Fig. \ref{fig:pend_wm_res}) is virtually identical to the ground-truth optimization, validating the learned model's accuracy for long-horizon planning ($H=80$).

\begin{figure}[H]
    \centering
        \begin{minipage}{0.48\textwidth}
        \centering
        \includegraphics[width=\linewidth]{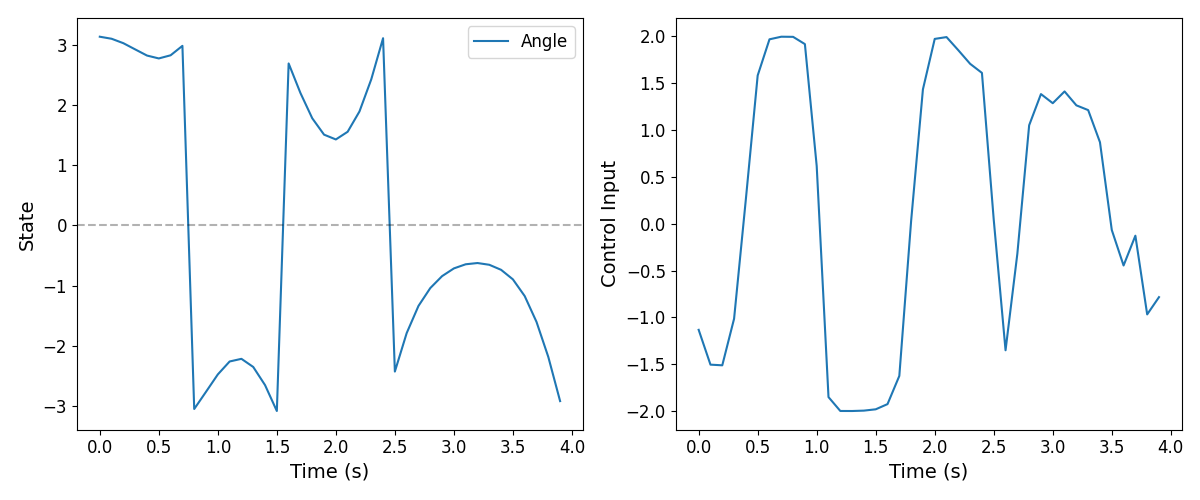}
        \caption{World Model: Generated Trajectory.}
        \label{fig:pend_wm_res}
    \end{minipage}
    \begin{minipage}{0.48\textwidth}
        \centering
        \includegraphics[width=\linewidth]{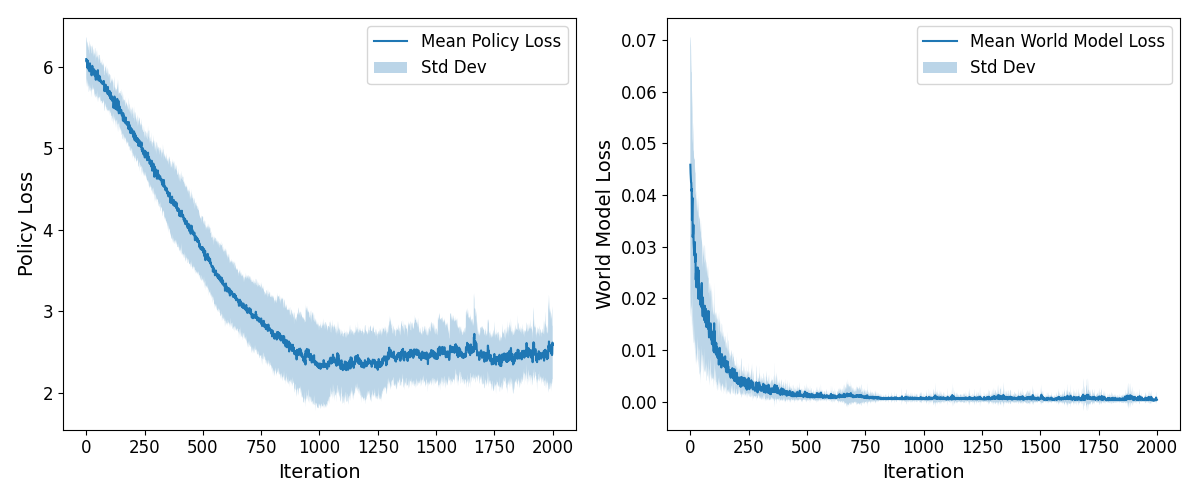}
        \caption{Joint Training: Policy Loss (left) vs World Model Loss (right).}
        \label{fig:pend_wm_loss}
    \end{minipage}\hfill
\end{figure}

\subsubsection{High-Dimensional Coupled Oscillators Results (Natural Gradient optimizer)}

For the high-dimensional coupled oscillators, the grid-based method (Algorithm 1) becomes numerically intractable due to the curse of dimensionality. Therefore, we restrict our comparison to neural methods, evaluating performance with and without a differentiable environment.

Figures \ref{fig:high_dim_diffphys} and \ref{fig:high_dim_diffphys_loss} present results using ground-truth dynamics. Initialized with alternating displacements ($\pm 1$), the system is successfully stabilized to the origin within approximately 2 seconds. The control inputs (bottom plot) exhibit a sharp initial reaction followed by a smooth decay to zero. The monotonic decrease in loss (Figure \ref{fig:high_dim_diffphys_loss}) further indicates stable optimization of the trajectory cost.

\begin{figure}[H]
    \centering
    \includegraphics[width=0.9\textwidth]{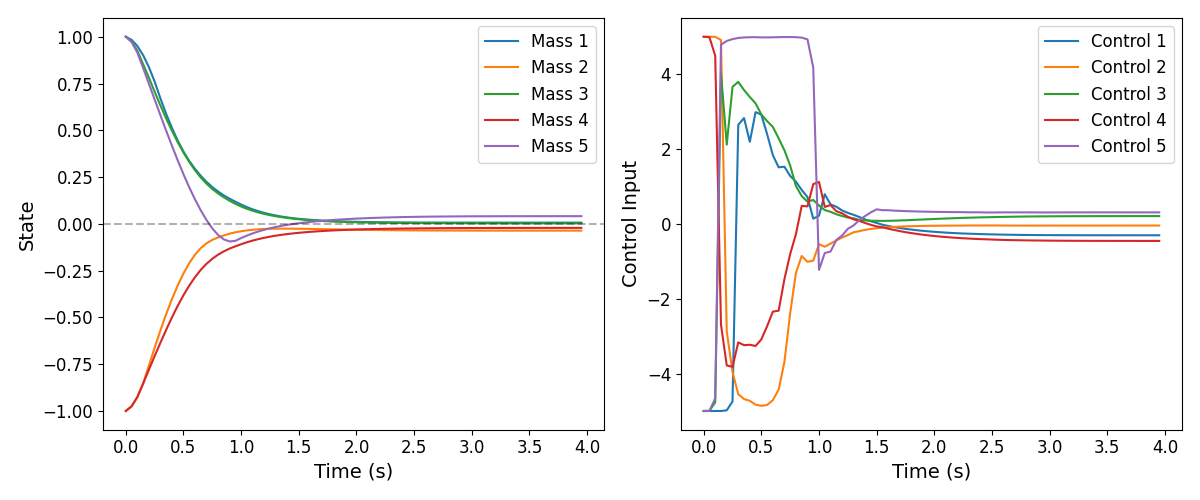}
    \caption{Direct Differentiable Physics: State trajectories (top) and Control inputs (bottom).}
    \label{fig:high_dim_diffphys}
\end{figure}

\begin{figure}[H]
    \centering
    \includegraphics[width=0.6\textwidth]{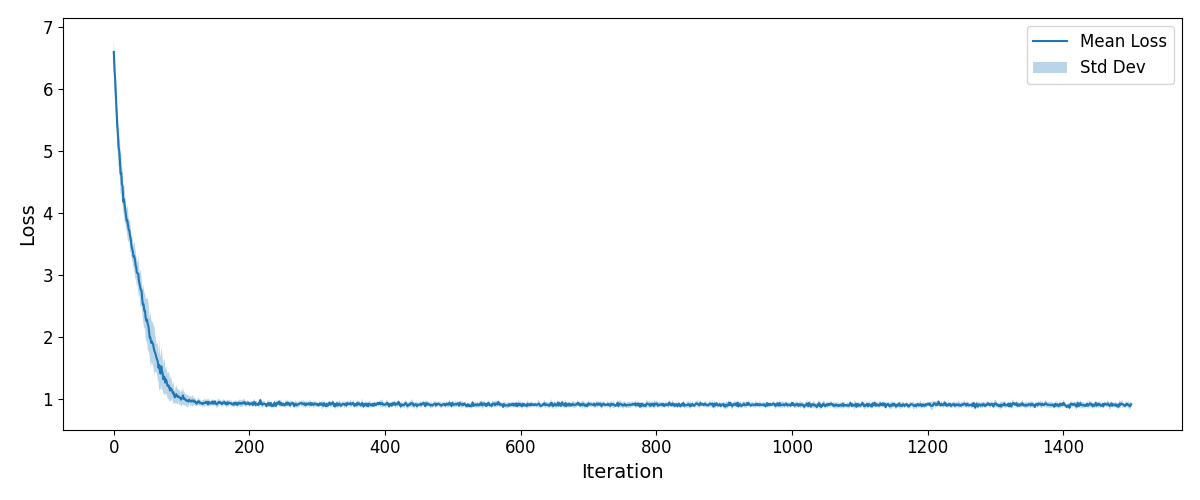}
    \caption{Training Loss for Direct Differentiable Physics.}
    \label{fig:high_dim_diffphys_loss}
\end{figure}

Finally, Figures \ref{fig:high_dim_wm} and \ref{fig:high_dim_wm_loss} display results when optimizing through a learned World Model. Despite relying on an approximation of the complex coupled dynamics, the agent achieves stabilization performance comparable to the baseline, with masses converging effectively to zero. The joint training process is illustrated in Figure \ref{fig:high_dim_wm_loss}: the \textit{World Model Loss} (right) decreases rapidly, showing that the physics $s_{t+1} \approx f_\phi(s_t, a_t)$ is being captured, while the \textit{Policy Loss} (left) decreases as the planner exploits this model to minimize cost.

\begin{figure}[H]
    \centering
    \includegraphics[width=0.9\textwidth]{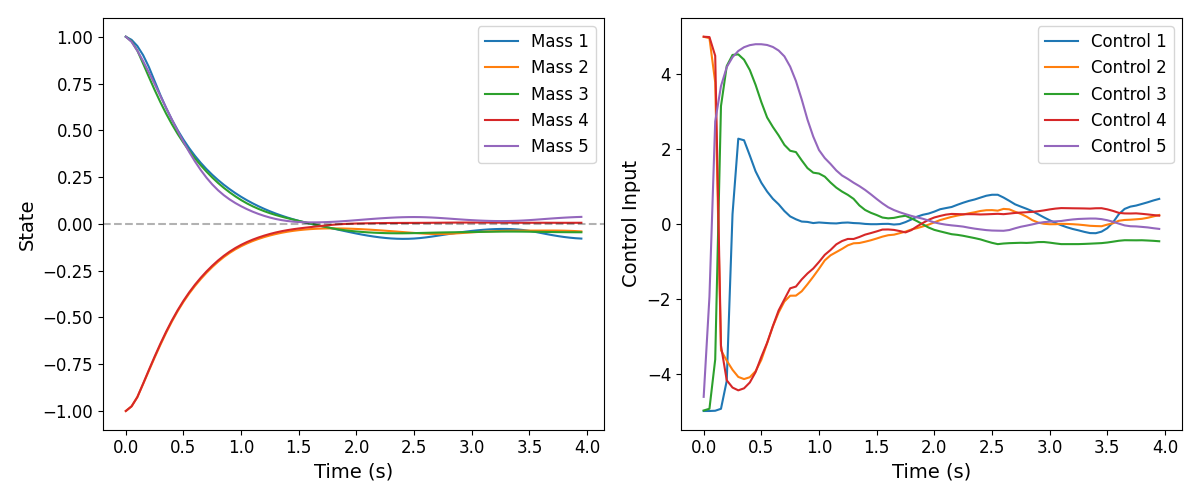}
    \caption{World Model Approach: State trajectories (left) and Control inputs (right).}
    \label{fig:high_dim_wm}
\end{figure}

\begin{figure}[H]
    \centering
    \includegraphics[width=0.9\textwidth]{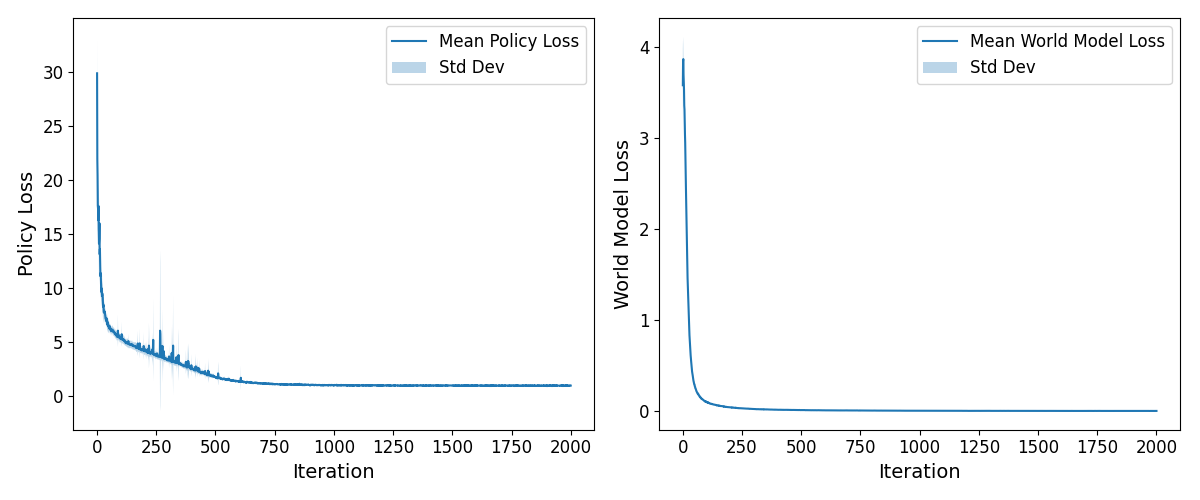}
    \caption{Joint Training Losses: Policy Loss (left) and World Model Prediction Loss (right).}
    \label{fig:high_dim_wm_loss}
\end{figure}
\bibliographystyle{plain}
\bibliography{biblio}

\end{document}